\documentclass[11pt,letterpaper]{article}
\usepackage{amsmath,amssymb,amsthm}
\usepackage{MnSymbol} 
\usepackage{graphicx,color}
\usepackage[hyphens]{url}
\usepackage{dsfont}
\usepackage{booktabs}
\usepackage[normalem]{ulem}
\usepackage{mathtools}
\usepackage[nameinlink,capitalize]{cleveref}
\usepackage{multirow}
\usepackage{algorithm}
\usepackage[noend]{algpseudocode}
\usepackage{tikz} 

\usepackage[square,sort,numbers]{natbib}
\usepackage[sort,nocompress,noadjust]{cite}

\usepackage{subcaption}

\usepackage{tabularx}
\newcolumntype{L}{>{\raggedright\arraybackslash}X}

\numberwithin{equation}{section}
\newtheorem{theorem}{Theorem}[section]

\newtheorem{lemma}[theorem]{Lemma}
\newtheorem{remark}[theorem]{Remark}
\newtheorem{prop}[theorem]{Proposition}
\newtheorem{obs}[theorem]{Observation}

\newtheorem{question}[theorem]{Question}
\newtheorem{example}[theorem]{Example}

\newtheorem{defin}[theorem]{Definition}

\newcommand{\cA}{\mathcal{A}}

\newcommand{\cD}{\mathcal{D}}

\newcommand{\cK}{\mathcal{K}}

\newcommand{\cN}{\mathcal{N}}

\newcommand{\cP}{\mathcal{P}}

\newcommand{\cW}{\mathcal{W}}
\newcommand{\cX}{\mathcal{X}}

\newcommand{\R}{\mathbb{R}}

\newcommand{\N}{\mathbb{N}}

\newcommand{\Prob}{\mathbb{P}}
\newcommand*{\E}{\mathbb{E}}
\newcommand{\supp}{\operatorname{supp}}

\newcommand*{\Otilde}{\tilde{O}}

\DeclareMathOperator{\logdel}{\log\tfrac{1}{\delta}}

\newcommand*{\eps}{\varepsilon}
\newcommand*{\del}{\delta}

\newcommand{\plusminus}{\raisebox{.2ex}{$\scriptstyle\pm$}}

\renewcommand{\leq}{\leqslant}
\renewcommand{\geq}{\geqslant}
\DeclareMathOperator{\half}{\frac{1}{2}}

\providecommand{\norm}[1]{\lVert{#1}\rVert}

\newcommand{\proj}{\Pi_{\cK}}

\newcommand{\sig}{\sigma}

\newcommand{\Tthresh}{\bar{T}}

\newcommand{\gradf}{\nabla f}

\newcommand{\Dalp}[2]{\cD_{\alpha}(#1 \; \| \; #2)}
\newcommand{\Dalplr}[2]{\cD_{\alpha}\left(#1 \; \| \; #2\right)}

\newcommand{\Dalpshiftlr}[3]{\cD_{\alpha}^{(#1)}\left(#2 \; \| \; #3\right)}

\newcommand{\CNI}{\mathrm{CNI}}

\newcommand{\SGD}{\textsf{SGD}}

\newcommand{\SGLD}{\textsf{SGLD}}

\newcommand{\IGD}{\textsf{Noisy-GD}}
\newcommand{\ISGD}{\textsf{Noisy-SGD}}
\newcommand{\CSGD}{\textsf{Noisy-CGD}}

\newcommand*\circled[1]{\tikz[baseline=(char.base)]{
		\node[shape=circle,draw,inner sep=2pt] (char) {#1};}}

\usepackage[margin=1in]{geometry}

\makeatletter
\def\blfootnote{\gdef\@thefnmark{}\@footnotetext}
\makeatother

\begin{document}

	\title{Privacy of Noisy Stochastic Gradient Descent: \\ More Iterations without More Privacy Loss}\blfootnote{
		A preliminary conference version of this paper appeared in NeurIPS 2022~\citep{AltTal22privacyneurips}.
	}
	
	\author{
		Jason M. Altschuler\footnote{This work was done while JA was an intern at Apple during Summer 2021. JA was also supported in part by NSF Graduate Research Fellowship 1122374 and a TwoSigma PhD Fellowship.}\\
		MIT\\
		\texttt{jasonalt@mit.edu}
		\and
		Kunal Talwar \\
		Apple \\
		\texttt{ktalwar@apple.com}
	}
	\date{}
	\maketitle


\begin{abstract}
A central issue in machine learning is how to train models on sensitive user data. Industry has widely adopted a simple algorithm: Stochastic Gradient Descent with noise (a.k.a. Stochastic Gradient Langevin Dynamics). However, foundational theoretical questions about this algorithm's privacy loss remain open---even in the seemingly simple setting of smooth convex losses over a bounded domain. Our main result resolves these questions: for a large range of parameters, we characterize the differential privacy up to a constant factor. This result reveals that all previous analyses for this setting have the wrong qualitative behavior. Specifically, while previous privacy analyses increase ad infinitum in the number of iterations, we show that after a small burn-in period, running SGD longer leaks no further privacy. 

Our analysis departs from previous approaches based on fast mixing, instead using techniques based on optimal transport (namely, Privacy Amplification by Iteration) and the Sampled Gaussian Mechanism (namely, Privacy Amplification by Sampling). Our techniques readily extend to other settings, e.g., strongly convex losses, non-uniform stepsizes, arbitrary batch sizes, and random or cyclic choice of batches.
\end{abstract}

	\setcounter{tocdepth}{2}	
	\tableofcontents
	

\section{Introduction}

Convex optimization is a fundamental task in machine learning. When models are learnt on sensitive data, privacy becomes a major concern---motivating a large body of work on differentially private convex optimization~\citep{ChaudhuriMoSa11,bassily2014private,BassilyFeTaTh19,feldman2020private, KiferSmTh12,JainTh14,TalwarThZh15,BassilyGuNa21}. In practice, the most common approach for training private models is $\ISGD$, i.e., Stochastic Gradient Descent with noise added each iteration. This algorithm is simple and natural, has optimal utility bounds~\citep{bassily2014private,BassilyFeTaTh19}, and is implemented on mainstream machine learning platforms such as Tensorflow (TF Privacy)~\citep{abadi2016tensorflow}, PyTorch (Opacus)~\citep{yousefpour2021opacus}, and JAX~\citep{jax2018github}. 

\par Yet, despite the simplicity and ubiquity of this $\ISGD$ algorithm, we do not understand basic questions about its \textit{privacy loss}\footnote{We follow the literature by writing ``privacy loss'' to refer to the (R\'enyi) differential privacy parameters.}---i.e., how sensitive the output of $\ISGD$ is with respect to the training data. Specifically:
\begin{question}\label{q:full}
	What is the privacy loss of $\ISGD$ as a function of the number of iterations?
\end{question}
Even in the seemingly simple setting of smooth convex losses over a bounded domain, this fundamental question has remained wide open. In fact, even more basic questions are open: 

\begin{question}\label{q:conv}
	Does the privacy loss of $\ISGD$ increase ad infinitum in the number of iterations? 
\end{question}

The purpose of this paper is to understand these fundamental theoretical questions. Specifically, we resolve Questions~\ref{q:full} and~\ref{q:conv} by characterizing the privacy loss of $\ISGD$ up to a constant factor in this (and other) settings for a large range of parameters. Below, we first provide context by describing previous analyses in \S\ref{ssec:intro:prev-1}, and then describe our result in \S\ref{ssec:intro:cont} and techniques in \S\ref{ssec:intro:tech}.

\subsection{Previous approaches and limitations}\label{ssec:intro:prev-1}
Although there is a large body of research devoted to understanding the privacy loss of $\ISGD$, all existing analyses have (at least) one of the following two drawbacks. 

\paragraph*{Privacy bounds that increase ad infinitum.} One type of analysis approach yields upper bounds on the DP loss (resp., R\'enyi DP loss) that scale in the number of iterations $T$ as $\sqrt{T}$ (resp., $T$). This includes the original analyses of $\ISGD$, which were based on the techniques of Privacy Amplification by Sampling and Advanced Composition~\citep{bassily2014private, abadi2016deep}, as well as alternative analyses based on the technique of Privacy Amplification by Iteration~\citep{pabi}. A key issue with these analyses is that they increase unboundedly in $T$. This limits the number of iterations that $\ISGD$ can be run given a reasonable privacy budget, typically leading to suboptimal optimization error in practice. Is this a failure of existing analysis techniques or an inherent fact about the privacy loss of $\ISGD$?

\paragraph*{Privacy bounds that apply for large $T$ and require strong additional assumptions.} 
The second type of analysis approach yields convergent upper bounds on the privacy loss, but requires strong additional assumptions. This approach is based on connections to sampling. The high-level intuition is that $\ISGD$ is a discretization of a continuous-time algorithm with bounded privacy loss. Specifically, $\ISGD$ can be interpreted as the Stochastic Gradient Langevin Dynamics ($\SGLD$) algorithm~\citep{WellingT11}, which is a discretization of a continuous-time Markov process whose stationary distribution is equivalent to the exponential mechanism~\citep{McSherryT07} and thus is differentially private under certain assumptions.  

\par However, making this connection precise requires strong additional assumptions and/or the resolution of longstanding open questions about the mixing time of $\SGLD$ (see \S\ref{ssec:intro:rel} for details). Only recently did a large effort in this direction culminate in the breakthrough work by~\citet{ChourasiaYS21}, which proves that {\em full batch}\footnote{Recently, \citet{YeS22} and \citet{RyffelBP22}, in works concurrent to the present paper, extended this result of~\citep{ChourasiaYS21} to $\SGLD$ by removing the full batch assumption; we also obtain the same result by a direct extension of our (completely different) techniques, see \S\ref{app:sc}. Note that both these papers~\citep{YeS22,RyffelBP22} still require strongly convex losses, and in fact state in their conclusions that the removal of this assumption is an open problem that ``would pave the way for wide adoption by data scientists.'' Our main result resolves this question.\label{fn:sc}} Langevin dynamics (a.k.a., $\IGD$ rather than $\ISGD$) has a privacy loss that converges as $T \to \infty$ in this setting where the smooth losses are additionally assumed to be \emph{strongly convex}.

Unfortunately, the assumption of strong convexity seems unavoidable with current techniques. Indeed, in the absence of strong convexity, it is not even known if $\ISGD$ converges to a private stationary distribution, let alone if this convergence occurs in a reasonable amount of time. (The tour-de-force work~\citep{bubeck2018sampling} shows mixing in (large) polynomial time, but only in total variation distance which does not have implications for privacy.)
There are fundamental challenges for proving such a result. In short, $\SGD$ is only a weakly contractive process without strong convexity, which means that its instability increases with the number of iterations~\citep{HardtRS16}---or in other words, it is plausible that $\ISGD$ could run for a long time while memorizing training data, which would of course mean it is not a privacy-preserving algorithm. As such, given state-of-the-art analyses in both sampling and optimization, it is unclear if the privacy loss of $\ISGD$ should even remain bounded; i.e., it is unclear what answer one should even expect for Question~\ref{q:conv}, let alone Question~\ref{q:full}.

\subsection{Contributions}\label{ssec:intro:cont}

The purpose of this paper is to resolve Questions~\ref{q:full} and~\ref{q:conv}. To state our result requires first recalling the parameters of the problem. Throughout, we prefer to state our results in terms of R\'enyi Differential Privacy (RDP); these RDP bounds are easily translated to DP bounds, as mentioned below in Remark~\ref{rem:iid-dp}. See the preliminaries section \S\ref{sec:prelim} for definitions of DP and RDP. 

\par We consider the basic $\ISGD$ algorithm run on a dataset $\cX = \{x_1, \dots, x_n\}$, where each $x_i$ defines a convex, $L$-Lipschitz, and $M$-smooth loss function $f_i(\cdot)$ on a convex set $\cK$ of diameter $D$. For any stepsize $\eta \leq 2/M$, batch size $b$, and initialization $\omega_0 \in \cK$, we iterate $T$ times the update 
\[
\omega_{t+1} \gets \proj\big[\omega_t -\eta ( G_t  + Z_t)\big],
\]
where $G_t$ denotes the average gradient vector on a random batch of size $b$, $Z_t \sim \cN(0, \sigma^2 I_d)$ is an isotropic Gaussian\footnote{Note that different papers have different notational conventions for the total noise $\eta Z$: the variance is $\eta^2 \sig^2$ here and, e.g.,~\citep{pabi}; but is $\eta \sig^2$ in other papers like~\citep{ChourasiaYS21}; and is $\sig^2$ in others like~\citep{bassily2014private}. To translate results, simply rescale $\sig$.}, and $\proj$ denotes the Euclidean projection onto $\cK$.

\par Known privacy analyses of $\ISGD$ give an ($\alpha,\eps)$-RDP upper bound of
\begin{align}
	\eps \lesssim \frac{\alpha L^2}{n^2\sigma^2} T
	\label{eq:simple-bound}
\end{align}
which increases ad infinitum as the number of iterations $T \to \infty$~\citep{bassily2014private, abadi2016deep,pabi}. Our main result is a tight characterization of the privacy loss, answering Question~\ref{q:full}. This result also answers Question~\ref{q:conv} and shows that previous analyses have the wrong qualitative behavior: after a small burn-in period of $\Tthresh \asymp \tfrac{nD}{L\eta}$ iterations, $\ISGD$ leaks no further privacy. See Figure~\ref{fig:main} for an illustration.

\begin{figure}
	\centering
	\includegraphics[width=0.7\linewidth]{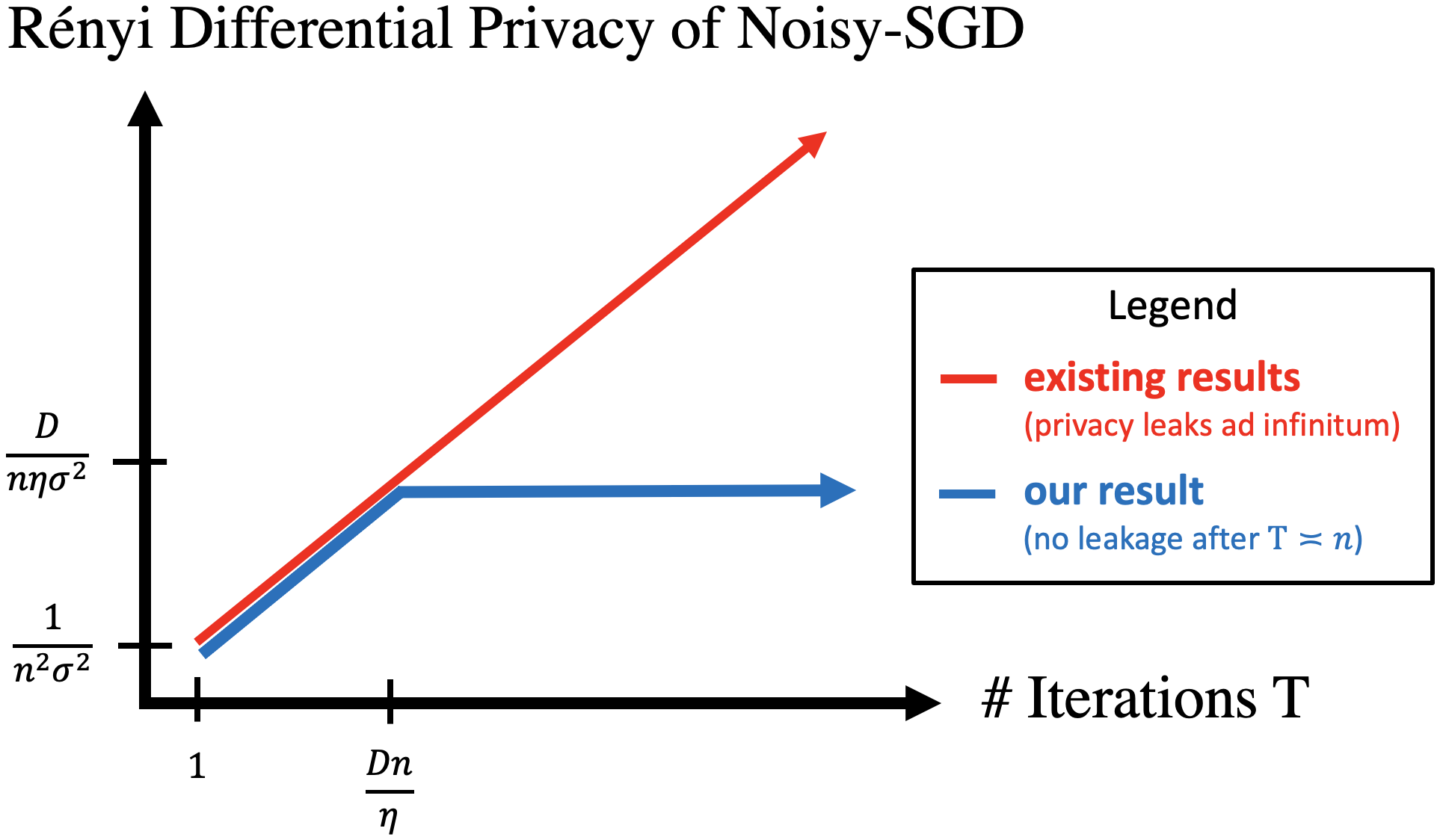}
	\caption{\footnotesize Even in the basic setting of smooth convex optimization over a bounded domain, existing implementations and theoretical analyses of $\ISGD$---the standard algorithm for private optimization---leak privacy ad infinitum as the number of iterations $T$ increases. Our main result Theorem~\ref{thm:iid} establishes that after a small burn-in period of $\Tthresh \asymp n$ iterations, there is no further privacy loss. For simplicity, this plot considers Lipschitz parameter $L = 1$ (wlog by rescaling), R\'enyi parameter $\alpha$ of constant size (the regime in practice), and omits logarithmic factors; see the main text for our precise (and optimal) dependence on all parameters.} \label{fig:main}
\end{figure}

\begin{theorem}[Informal statement of main result: tight characterization of the privacy of \ISGD]\label{thm:iid}
	For a large range of parameters, $\ISGD$ satisfies $(\alpha,\eps)$-RDP for
	\begin{align}
		\eps \lesssim
		\frac{\alpha L^2}{n^2\sig^2} \min\left\{T, \frac{Dn}{L\eta}\right\},
		\label{eq:iid:eps-rdp}
	\end{align}
	and moreover this bound is tight up to a constant factor.
\end{theorem}

Observe that the privacy bound in Theorem~\ref{thm:iid} is identical to the previous bound~\eqref{eq:simple-bound} when the number of iterations $T$ is small, but stops increasing when $T \geq \Tthresh$. Intuitively, $\Tthresh$ can be interpreted as the smallest number of iterations required for two $\SGD$ processes run on adjacent datasets to, with reasonable probability, reach opposite ends of the constraint set.\footnote{Details: $\ISGD$ updates on adjacent datasets differ by at most $\eta L/b$ if the different datapoint is in the current batch (i.e., with probability $b/n$), and otherwise are identical. Thus, in expectation and with high probability, it takes $\Tthresh \asymp (Dn)/(\eta L)$ iterations for two $\ISGD$ processes to differ by distance $D$.} In other words, $\Tthresh$ is effectively the smallest number of iterations before the final iterates could be maximally distinguishable.

To prove Theorem~\ref{thm:iid}, we show a matching upper bound (in \S\ref{sec:iid}) and lower bound (in \S\ref{sec:lb}). These two bounds are formally stated as Theorems~\ref{thm:iid-ub} and~\ref{thm:iid-lb}, respectively. See \S\ref{ssec:intro:tech} for an overview of our techniques and how they depart from previous approaches.

\par We conclude this section with several remarks about Theorem~\ref{thm:iid}.

\begin{remark}[Tight DP characterization]\label{rem:iid-dp}
	While Theorem~\ref{thm:iid} characterizes the privacy loss of $\ISGD$ in terms of RDP, our results can be restated in terms of standard DP bounds. Specifically, by a standard conversion from RDP to DP (Lemma~\ref{lem:rdp-to-dp}), if $\delta$ is not smaller than exponentially small in $-b\sigma/n$ and if the resulting bound is $\eps \lesssim 1$, 
	it follows that $\ISGD$ is $(\eps,\del)$-DP for 
	\begin{align}
		\eps \lesssim
		\frac{L}{n\sig}
		\sqrt{ 
			\min\left\{T, \frac{D n}{L\eta}\right\} \log 1/\delta}.
		\label{eq:iid:eps-dp}
	\end{align}
	A matching DP lower bound, for the same large range of parameters as in Theorem~\ref{thm:iid}, is proved along the way when we establish our RDP lower bound in \S\ref{sec:lb}.
\end{remark}

\begin{figure}
	\hspace*{\fill} 
	\begin{subfigure}[t]{0.45\textwidth}
		\includegraphics[width=\linewidth]{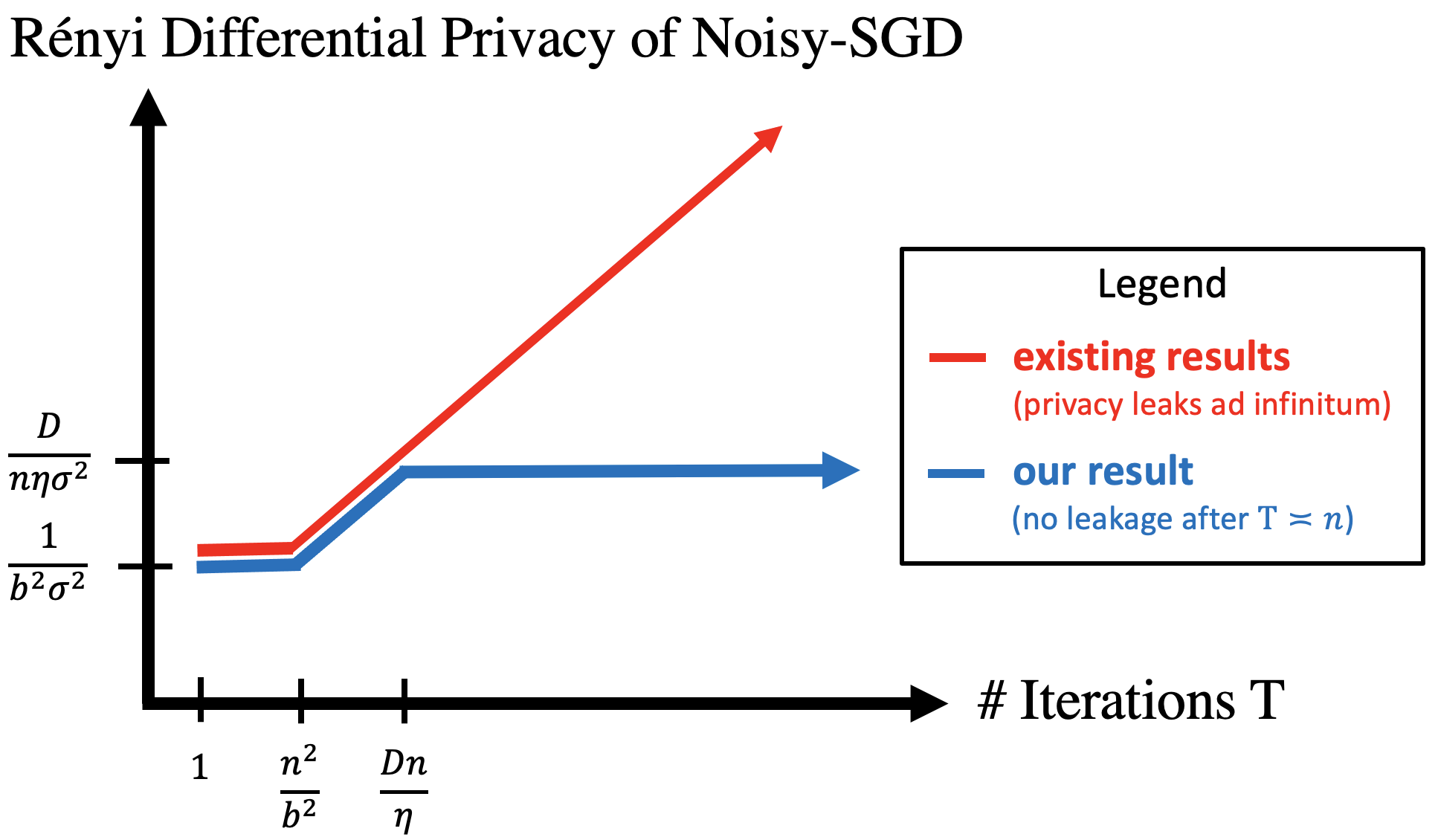}
		\caption{\footnotesize In the setting of \emph{cyclic batch updates}, the privacy of $\SGD$ is identical except for an unavoidable worsening in an initial phase (Theorem~\ref{thm:diam-rdp}). 
		} \label{fig:cyc}
	\end{subfigure}
	\hspace*{\fill} 
	\begin{subfigure}[t]{0.45\textwidth}
		\includegraphics[width=\linewidth]{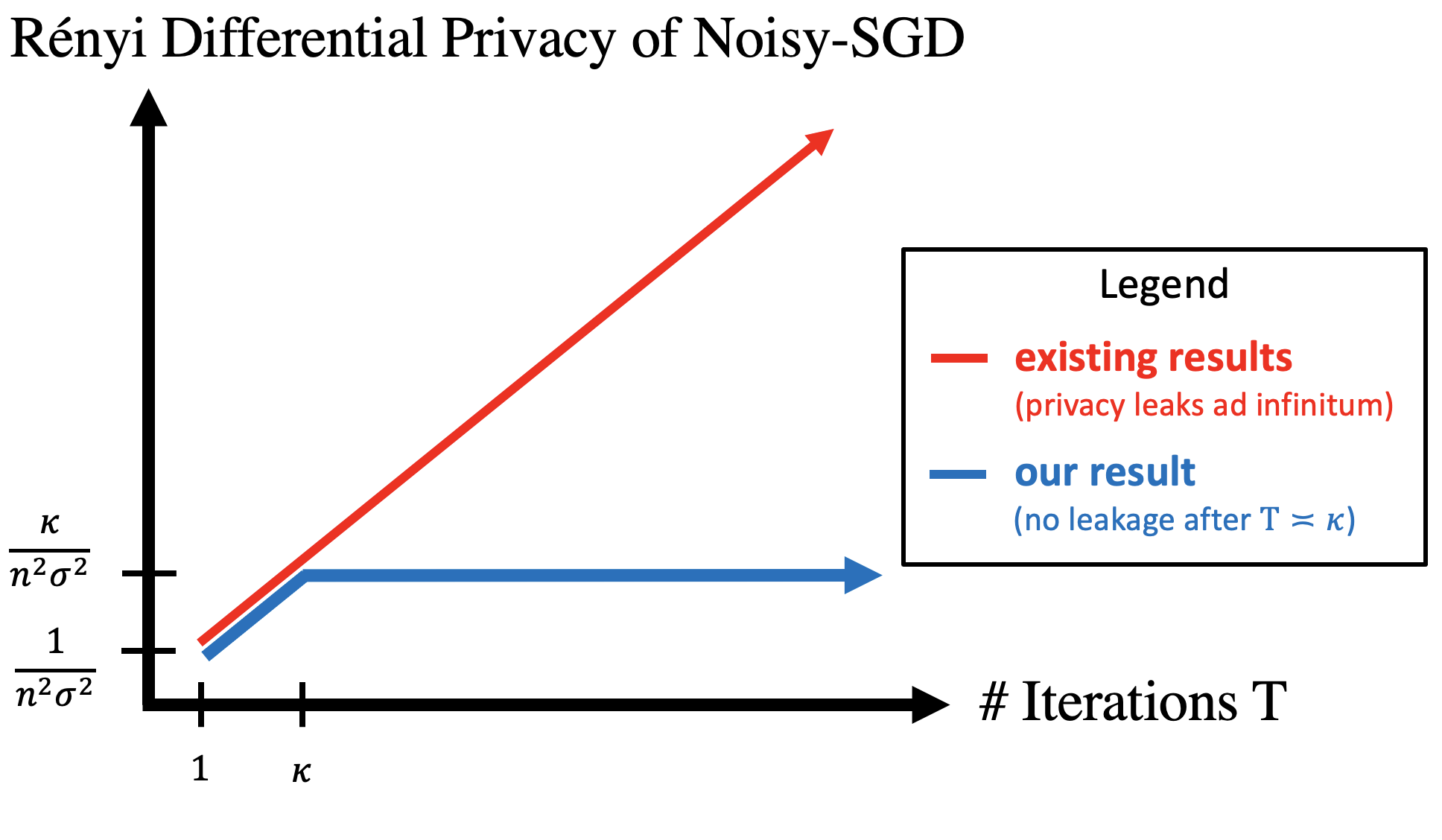}
		\caption{\footnotesize In the setting of \emph{strongly convex losses}, the privacy of $\ISGD$ is identical except the threshold $\Tthresh$ improves (Theorem~\ref{thm:iid-sc}). Plotted here for $\eta = 1/M$.} \label{fig:scs}
	\end{subfigure}
	\caption{
		\footnotesize
		Analog of Figure~\ref{fig:main} for the extensions of our main result to other settings.
	}
	\label{fig:extensions}
\end{figure}

\paragraph*{Discussion of extensions.} Our analysis techniques immediately extend to other settings:
\begin{itemize}
	\item \underline{Strongly convex losses.} If the convexity assumption on the losses is replaced by $m$-strong convexity, then $\ISGD$ enjoys better privacy. Specifically, the bound in Theorem~\ref{thm:iid} extends identially except that the threshold $\Tthresh$ for no further privacy loss improves from linear to \emph{logarithmic} in $n$, namely $\Otilde(1/(\eta m))$. For example, if the stepsize $\eta = 1/M$, then $\Tthresh = \Otilde(\kappa)$, where $\kappa = M/m$ denotes the condition number of the losses. This matches the independent results of~\citep{RyffelBP22,YeS22} (see Footnote~\ref{fn:sc}) and uses completely different techniques from them. Details in \S\ref{app:sc}; see also Figure~\ref{fig:scs} for an illustration of the privacy bound.
	\item \underline{Cyclic batch updates.} $\ISGD$ is sometimes run with batches that are chosen cyclicallly rather than randomly. The bound in Theorem~\ref{thm:iid} extends identically, except for an initial transient phase in which there is unavoidably a large privacy loss since the algorithm may choose to update using a sensitive datapoint in an early batch. Details in \S\ref{app:cyclic}; see also Figure~\ref{fig:cyc} for an illustration of the privacy bound.
	\item \underline{Non-uniform stepsizes.} $\ISGD$ is sometimes run with decaying stepsizes $\eta_t = t^{-c}$ for $c \in (0,1)$. The bound in Theorem~\ref{thm:iid} extends identically, with $\eta$ replaced by $\eta_T$. This means that the privacy increases after the threshold, but at the slower rate of $T^c \ll T$ for RDP (resp., $T^{c/2} \ll T^{1/2}$ for DP). Details in \S\ref{app:nonuniform}.
\end{itemize}

\paragraph*{Discussion of assumptions.}
\begin{itemize}
		\item \underline{Boundedness.} All previous privacy analyses are oblivious to any sort of diameter bound and therefore unavoidably increase ad infinitum\footnote{For strongly convex losses, boundedness is unnecessary for convergent privacy; see \S\ref{app:sc}.}, c.f.\,our lower bound in Theorem~\ref{thm:iid}. Our analysis is the first to exploit boundedness: whereas previous analyses can only argue that having a smaller constraint set does not worsen the privacy loss, we show that this strictly improves the privacy loss, in fact making it convergent as $T \to \infty$. See the techniques section \S\ref{ssec:intro:tech}.
	\par We emphasize that this diameter bound is a mild constraint and is in common with all existing theoretical bounds for utility/optimization. Indeed, every utility guarantee for (non-strongly) convex losses also has a dependence on some similar diameter bound; this is inevitable simply due to the difference between the initialization and optimum. We also mention that one can solve an unconstrained problem by solving constrained problems with norm bounds, paying only a small logarithmic overhead on the number of solves and a constant overhead in the privacy loss using known techniques~\citep{liu2019private}. Moreover, in many optimization problems, the solution set is naturally constrained either from the problem formulation or application. 
	\item \underline{Smoothness.} The smoothness assumption on the losses can be relaxed if one runs $\ISGD$ on a smoothed version of the objective. This can be done using standard smoothing techniques, e.g., Gaussian convolution smoothing~\citep[\S5.5]{pabi}, or Moreau-Yousida smoothing by replacing gradient steps with proximal steps~\citep[\S4]{BassilyFeTaTh19}.
	\item \underline{Convexity.}
	Convexity appears to be essential for privacy bounds that do not increase ad infinitum in $T$.
	Our analysis exploits convexity to ensure that $\ISGD$ is a contractive process. However, convexity is too restrictive when training deep networks, and it is an interesting question if the assumption of convexity can be relaxed. Any such result appears to require new techniques---if even true. The key challenge is that for any iterative process whose set of reachable fixed points is non-convex, there will be non-contractivity at the boundary between basins of attraction---and this appears to preclude arguments based on Privacy Amplification by Iteration, which at present is the only non-sampling-based technique with a hope of establishing convergent privacy bounds (see the techniques section \S\ref{ssec:intro:tech}).
	\item \underline{Lipschitzness.} This assumption can be safely removed as the (necessary) assumptions of smoothness and bounded diameter already imply Lipschitzness (with parameter $L=MD$). We write our result in this way in order to isolate where this dependence comes from more simply, and also because the Lipschitz parameter $L$ may be much better than $MD$.
	\par Our analysis only uses Lipschitzness to bound gradient sensitivity, just like all previous analyses that use Privacy Amplification by Sampling. For details on this, as well as an example where gradient sensitivity is more useful in practice than Lipschitzness, see Section \S\ref{app:reg}.
	%
	%
	%
	\item \underline{Mild assumptions on parameters.} The lower bound in Theorem~\ref{thm:iid} makes the mild assumption that the diameter $D$ of the decision set is asymptotically larger than both the movement size from one gradient step, and the standard deviation from $\Tthresh$ random increments of Gaussian noise $\cN(0,\eta^2\sig^2)$ so that the noise does not overwhelm the learning.
	\par The upper bound uses the technique of Privacy Amplification by Sampling, and thus inherits two mild assumptions on the parameters that are also inherent in all previous analyses of $\ISGD$. One is an upper bound assumption on $\alpha$, and the other is a lower bound assumption on the noise $\sig$ so that one step of $\ISGD$ provides at least constant privacy to any user. These restrictions neither affect numerical bounds which can be computed for any $\alpha$ and $\sig$ (see \S\ref{ssec:prelim:pabs}), nor do they affect the asymptotic $(\eps,\delta)$-DP bounds in most parameter regimes of interest. We also remark that our results require no restriction if the batches are chosen cyclically, and thus also if full-batch gradients are used (see Theorem~\ref{thm:diam-rdp}).
\end{itemize}

\subsection{Techniques}\label{ssec:intro:tech}

\paragraph*{Upper bound on privacy.} Our analysis technique departs from recent approaches based on fast mixing. This allows us to bypass the many longstanding technical challenges discussed in \S\ref{ssec:intro:prev-1}.

\par Instead, our analysis combines the techniques of Privacy Amplification by Iteration and Privacy Amplification by Sampling. As discussed in \S\ref{ssec:intro:prev-1}, previous analyses based on these techniques yield loose privacy bounds that increase ad infinitum in $T$. Indeed, bounds based on Privacy Amplification by Sampling inevitably diverge since they ``pay'' for releasing the entire sequence of $T$ iterates, each of which leaks more information about the private data~\citep{bassily2014private, abadi2016deep}. Privacy Amplification by Iteration avoids releasing the entire sequence by directly arguing about the final iterate; however, previous arguments based on this technique are unable to exploit a diameter bound on the constraint set, and therefore inevitably lead to privacy bounds which grow unboundedly in $T$~\citep{pabi}.

\par Our analysis hinges on the observation that when the iterates are constrained to a bounded domain, we can combine the techniques of Privacy Amplification by Iteration and Privacy Amplification by Sampling in order to only pay for the privacy loss from the final $\Tthresh \asymp (Dn)/(L\eta)$ iterations. To explain this, first recall that by definition, differential privacy of $\ISGD$ means that the final iterate of $\ISGD$ is similar when run on two (adjacent but) different datasets from the same initialization. 
At an intuitive level, we establish this by arguing about the privacy of the following three scenarios:
\begin{itemize}
	\item[(i)] Run $\ISGD$ for $T$ iterations on different datasets from same initialization.
	\item[(ii)]  Run $\ISGD$ for $\Tthresh$ iterations on different datasets from different initializations.
	\item[(iii)] Run $\ISGD$ for $\Tthresh$ iterations on same datasets from different initializations.
\end{itemize}
In order to analyze (i), which as mentioned is the definition of $\ISGD$ being differentially private, we argue that no matter how different the two input datasets are, the two $\ISGD$ iterates at iteration $T - \Tthresh$ are within distance $D$ of each other since they lie within the constraint set $\cK$ of diameter $D$. Thus in particular, the privacy of scenario (i) is at most the privacy of scenario (ii), which has initializations that can be arbitrarily different so long as they are within distance $D$ from each other. In this way, we arrive at a scenario which is independent of $T$. This is crucial for obtaining privacy bounds which do not diverge in $T$; however, previously privacy analyses could not proceed in this way because they could not argue about different initialization.

\par In order to analyze (ii), we use the technique of Privacy Amplification by Sampling---but only on these final $\Tthresh \ll T$ iterations. 
In words, this enables us to argue that the noisy gradients that $\ISGD$ uses in the last $\Tthresh$ iterations are indistinguishable up to a certain privacy loss (to be precise the RDP scales linearly in $\Tthresh$) despite the fact that they are computed on different input datasets. This enables us to reduce analyzing the privacy of scenario (ii) to scenario (iii). 

\par In order to analyze (iii), we use a new diameter-aware version of Privacy Amplification by Iteration (Proposition~\ref{prop:pabi-new}). In words, this shows that running the \emph{same} $\ISGD$ updates on two different initializations masks their initial difference due to both the noise and contraction maps that define each $\ISGD$ update. Of course, this indistinguishability improves the longer that $\ISGD$ is run (to be precise, we show that the RDP scales inversely in $\Tthresh$).\footnote{This step of the analysis has a natural interpretation in terms of the mixing time of the Langevin Algorithm to its biased stationary distribution; this connection and its implications are explored in detail in~\citep{AltTal22mix}.}

\par We arrive, therefore, at a privacy loss which is the sum of the bounds from the Privacy Amplification by Sampling argument in step (ii) and the Privacy Amplification by Iteration argument in step (iii). This privacy loss has a natural tradeoff in the parameter $\Tthresh$: the former bound leads to a privacy loss that increases in $\Tthresh$ (since it pays for the release of $\Tthresh$ noisy gradients), whereas the latter bound leads to a privacy loss that decreases in $\Tthresh$ (since more iterations of the same $\ISGD$ process enable better masking of different initializations). Balancing these two privacy bounds leads to the final choice $\Tthresh \asymp \tfrac{Dn}{L\eta}$. (We remark that this quantity is not just an artefact of our analysis, but in fact is the true answer for when the privacy loss stops increasing, as established by our matching lower bound.)

\par At a high-level, our analysis proceeds by carefully running these arguments in parallel.

\par We emphasize that this analysis is the first to show that the privacy loss of $\ISGD$ can strictly improve if the constraint set is made smaller. In contrast, previous analyses only argue that restricting the constraint set \emph{cannot worsen} the privacy loss (e.g., by using a post-processing inequality to analyze the projection step). The key technical challenge in exploiting a diameter bound is dealing with the complicated non-linearities that arise when interleaving projections with noisy gradient updates. Our techniques enable such an analysis.

\paragraph*{Lower bound on privacy.} We construct two adjacent datasets for which the corresponding $\ISGD$ processes are random walks---one symmetric, one biased---that are confined to an interval of length $D$. In this way, we reduce the question of how large is the privacy loss of $\ISGD$, to the question of how distinguishable is a constrained symmetric random walk from a constrained biased random walk. The core technical challenge is that the distributions of the iterates of these random walks are intractable to reason about explicitly---due to the highly non-linear interactions between the projections and random increments. Briefly, our key technique here is a way to modify these processes so that on one hand, their distinguishability is essentially the same, and the other hand, no projections occur with high probability---allowing us to explicitly compute their distributions and thus also their distinguishability. Details in \S\ref{sec:lb}.

\subsection{Other related work}\label{ssec:intro:rel}

\paragraph*{Private sampling.} The mixing time of (stochastic) Langevin Dynamics has been extensively studied in recent years starting with~\citep{Dalalyan16,DurmusM19}. A recent focus in this vein is analyzing mixing in more stringent notions of distance, such as the R\'enyi divergence~\citep{VempalaW19,GaneshT20}, in part because this is necessary for proving privacy bounds. In addition to the aforementioned results, several other works focus on sampling from the distribution $\exp(\eps \sum_i f_i(\omega)$ or its regularized versions from a privacy viewpoint.~\citep{MinamiASN16} proposed a mechanism of this kind that works for unbounded $f$ and showed $(\eps,\delta)$-DP. Recently,~\citep{GopiLL22} gave better privacy bounds for such a regularized exponential mechanism, and designed an efficient sampler based only on function evaluation. Also,~\citep{GaneshTU22} showed that the continuous Langevin Diffusion has optimal utility bounds for various private optimization problems.


\par As alluded to in \S\ref{ssec:intro:prev-1}, there are several core issues with trying to prove DP bounds for $\ISGD$ by directly combining ``fast mixing'' bounds with ``private once mixed'' bounds. First, mixing results typically do not apply, e.g., since DP requires mixing in stringent divergences like R\'enyi, or because realistic settings with constraints, stochasticity, and lack of strong convexity are difficult to analyze---indeed, understanding the mixing time for such settings is a challenging open problem. Second, even when fast mixing bounds do apply, directly combining them with ``private once mixed'' bounds unavoidably leads to DP bounds that are loose to the point of being essentially useless (e.g., the inevitable dimension dependence in mixing bounds to the stationary distribution of the continuous-time Langevin Diffusion would lead to dimension dependence in DP bounds, which should not occur---as we show). Third, even if a Markov chain were private after mixing, one cannot conclude from this that it is private beforehand---indeed, there are simple Markov chains which are private after mixing, yet are exponentially non-private beforehand~\citep{GT22markov}.

\paragraph*{Utility bounds.} In the field of private optimization, one separately analyzes two properties of algorithms: (i) the privacy loss as a function of the number of iterations, and (ii) the utility (a.k.a., optimization error) as a function of the number of iterations. These two properties can then be combined to obtain privacy-utility tradeoffs. The purpose of this paper is to completely resolve (i); this result can then be combined with any bound on (ii).

\par Utility bounds for $\SGD$ are well understood~\citep{ShwartzShSrSr09}, and these analyses have enabled understanding utility bounds for $\ISGD$ in empirical~\citep{bassily2014private} and population~\citep{BassilyFeTaTh19} settings. However, there is a big difference between the minimax-optimal utility bounds in theory versus what is desired in practice. Indeed, while in theory a single pass of $\ISGD$ achieves the minimax-optimal population risk~\citep{feldman2020private}, in practice $\ISGD$ benefits from running longer to get more accurate training. In fact, this divergence is even true and well-documented for non-private $\SGD$ as well, where one epoch is minimax-optimal in theory, but in practice more epochs help. Said simply, this is because typical problems are not worst-case problems (i.e., minimax-optimal theoretical bounds are typically not representative of practice). For these practical settings, in order to run $\ISGD$ longer, it is essential to have privacy bounds which do not increase ad infinitum. Our paper resolves this.

\subsection{Organization}\label{ssec:intro:org}

\cref{sec:prelim} recalls relevant preliminaries. Our main result Theorem~\ref{thm:iid} is proved in \cref{sec:iid} (upper bound) and \cref{sec:lb} (lower bound). \cref{app:extensions} contains extensions to strongly convex losses, cyclic batch updates, and non-uniform stepsizes. \S\ref{sec:discussion} concludes with future research directions motivated by these results. For brevity, several proofs are deferred to~\cref{app:deferred}. For simplicity, we do not optimize constants in the main text; this is done in~\cref{app:numerics}.


\section{Preliminaries}\label{sec:prelim}

In this section, we recall relevant preliminaries about convex optimization (\S\ref{ssec:prelim:convex}), differential privacy (\S\ref{ssec:prelim:rdp}), and two by-now-standard techniques for analyzing the privacy of optimization algorithms---namely, Privacy Amplification by Sampling (\S\ref{ssec:prelim:pabs}) and Privacy Amplification by Iteration (\S\ref{ssec:prelim:pabi}). 


\paragraph*{Notation.} We write $\Prob_X$ to denote the law of a random variable $X$, and $\Prob_{X|Y=y}$ to denote the law of $X$ given the event $Y=y$. We write $Z_{S:T}$ as shorthand for the vector concatenating $Z_S, \dots, Z_T$. We write $f_{\#} \mu$ to denote the pushforward of a distribution $\mu$ under a function $f$, i.e., the law of $f(X)$ where $X \sim \mu$. We abuse notation slightly by writing $f_{\#} \mu$ even when $f$ is a ``random function'', by which we mean a Markov transition kernel. We write $\mu \ast \nu$ to denote the convolution of probability distributions $\mu$ and $\nu$, i.e., the law of $X+Y$ where $X \sim \mu$ and $Y \sim \nu$ are independently distributed. We write $\lambda \mu + (1-\lambda) \nu$ to denote the mixture distribution that is $\mu$ with probability $\lambda$ and $\nu$ with probability $1 - \lambda$. We write $\cA(\cX)$ to denote the output of an algorithm $\cA$ run on input $\cX$; this is a probability distribution if $\cA$ is a randomized algorithm.

\subsection{Convex optimization}\label{ssec:prelim:convex}

Throughout, the loss function corresponding to a data point $x_i$ or $x_i'$ is denoted by $f_i$ or $f_i'$, respectively. While the dependence on the data point is arbitrary, the loss functions are assumed to be convex in the argument $\omega \in \cK$ (a.k.a., the machine learning model we seek to train). Throughout, the set $\cK$ of possible models is a convex subset of $\R^d$. In order to establish both optimization and privacy guarantees, two additional assumptions are required on the loss functions. Recall that a differentiable function $g : \cK \to \R^d$ is said to be:
\begin{itemize}
	\item $L$-Lipschitz if $\|g(\omega) - g(\omega')\| \leq L\|\omega - \omega'\|$ for all $\omega,\omega' \in \cK$.
	\item $M$-smooth if $\|\nabla g(\omega) - \nabla g(\omega')\| \leq M\|\omega - \omega'\|$ for all $\omega,\omega' \in \cK$. 
\end{itemize}
An important implication of smoothness is the following well-known fact; for a proof, see e.g.,~\citep{nesterov2003introductory}.

\begin{lemma}[Small gradient steps on smooth convex losses are contractions]\label{lem:gd-contraction}
	Suppose $f : \R^d \to \R$ is a convex, $M$-smooth function. For any $\eta \leq 2/M$, the mapping $\omega \mapsto \omega - \eta \nabla f(\omega)$ is a contraction.
\end{lemma}

\subsection{Differential privacy and R\'enyi differential privacy}\label{ssec:prelim:rdp}

Over the past two decades, differential privacy (DP) has become a standard approach for quantifying how much sensitive information an algorithm leaks about the dataset it is run upon. Differential privacy was first proposed in~\citep{DworkMcNiSm06}, and is now widely used in industry (e.g., at Apple~\citep{apple-dp}, Google~\citep{google-dp}, Microsoft~\citep{microsoft-dp}, and LinkedIn~\citep{linkedin-dp}), as well as in government data collection such as the 2020 US Census~\citep{censusprivacy}. In words, DP measures how distinguishable the output of an algorithm is when run on two ``adjacent'' datasets---i.e., two datasets which differ on at most one datapoint.\footnote{This is sometimes called ``replace adjacency''. The other canonical notion of dataset adjacency is ``remove adjacency'', which means that one dataset is the other plus an additional datapoint. All results in this paper extend immediately to both settings; for details see Remark~\ref{rem:adjacency}.} 

\begin{defin}[Differential Privacy]\label{def:dp}
	A randomized algorithm $\cA$ satisfies $(\eps,\del)$-DP if for any two adjacent datasets $\cX,\cX'$ and any measurable event $S$,
	\[
		\Prob\left[ \cA(\cX) \in S \right] \leq e^{\eps} \Prob\left[ \cA(\cX') \in S \right] + \del.
	\]
\end{defin}

In order to prove DP guarantees, we work with the related notion of R\'enyi Differential Privacy (RDP) introduced by~\citep{mironov2017renyi}, since RDP is more amenable to our analysis techniques and is readily converted to DP guarantees. To define RDP, we first recall the definition of R\'enyi divergence. 

\begin{defin}[R\'enyi divergence]
	The R\'enyi divergence between two probability measures $\mu$ and $\nu$ of order $\alpha \in (1,\infty)$ is
	\[
	\Dalplr{\mu}{\nu}
	=
	\frac{1}{\alpha - 1} \log \int \left( \frac{\mu(x)}{\nu(x)}\right)^{\alpha} \nu(x) dx,
	\]
	if $\mu \ll \nu$, and is $\infty$ otherwise. Here we adopt the standard convention that $0/0 = 0$ and $x/0 = \infty$ for $x > 0$. The R\'enyi divergences of order $\alpha \in  \{1,\infty\}$ are defined by continuity.
\end{defin}

\begin{defin}[R\'enyi Differential Privacy]\label{def:rdp}
	A randomized algorithm $\cA$ satisfies $(\alpha,\eps)$-RDP if for any two adjacent datasets $\cX,\cX'$,
	\[
		\Dalplr{ \cA(\cX) }{ \cA(\cX') } \leq \eps.
	\]
\end{defin}

It is straightforward to convert an RDP bound to a DP bound as follows~\citep[Proposition 1]{mironov2017renyi}.

\begin{lemma}[RDP-to-DP conversion]\label{lem:rdp-to-dp}
	Suppose that an algorithm satisfies $(\alpha,\eps_{\alpha})$-RDP for some $\alpha > 1$. Then the algorithm satisfies $(\eps_{\del},\del)$-DP for any $\del \in (0,1)$ and $\eps_{\del} = \eps_{\alpha} + (\logdel)/(\alpha - 1)$.
\end{lemma}

Following, we recall three basic properties of RDP that we use repeatedly in our analysis. The first property regards convexity. While the KL divergence (the case $\alpha = 1$) is jointly convex in its two arguments, for $\alpha \neq 1$ the R\'enyi divergence is only jointly \emph{quasi}-convex. See e.g.,~\citep[Theorem 13]{van2014renyi} for a proof and a further discussion of partial convexity properties of the R\'enyi divergence.

\begin{lemma}[Joint quasi-convexity of R\'enyi divergence]\label{lem:qconvex-rdp}
	For any R\'enyi parameter $\alpha  \geq 1$, any mixing probability $\lambda \in [0,1]$, and any two pairs of probability distributions $\mu,\nu$ and $\mu',\nu'$,
	\[
		\Dalplr{\lambda \mu + (1-\lambda) \mu'}{ \lambda \nu + (1-\lambda) \nu' } 
		\leq 
		\max\left\{ \Dalplr{\mu}{\nu}, \Dalplr{\mu'}{\nu'} \right\}.
	\]
\end{lemma}

The second property states that pushing two measures forward through the same (possibly random) function cannot increase the R\'enyi divergence. This is the direct analog of the classic ``data-processing inequality'' for the KL divergence. See e.g.,~\citep[Theorem 9]{van2014renyi} for a proof.

\begin{lemma}[Post-processing property of R\'enyi divergence]\label{lem:process-rdp}
	For any R\'enyi parameter $\alpha \geq 1$, any (possibly random) function $h$, and any probability distributions $\mu,\nu$, 
	\[
		\Dalplr{h_{\#} \mu}{h_{\#} \nu}
		\leq
		\Dalplr{\mu}{\nu}.
	\]
\end{lemma}

The third property is the appropriate analog of the chain rule for the KL divergence (the direct analog of the chain rule does not hold for R\'enyi divergences, $\alpha \neq 1$). This bound has appeared in various forms in previous work (e.g.,~\citep{abadi2016deep,mironov2017renyi,dwork2016concentrated}); we state the following two equivalent versions of this bound since at various points in our analysis it will be more convenient to reference one or the other. A proof of the former version is in~\citep[Proposition 3]{mironov2017renyi}. The proof of the latter is similar to~\citep[Theorem 2]{abadi2016deep}; for the convenience of the reader, we provide a brief proof in~\cref{app:deferred:comp-rdp}.

\begin{lemma}[Strong composition for RDP, v1]\label{lem:comp-rdp-alg}
	Suppose algorithms $\cA_1$ and $\cA_2$ satisfy $(\alpha,\eps_1)$-RDP and $(\alpha,\eps_2)$-RDP, respectively. Let $\cA$ denote the algorithm which, given a dataset $\cX$ as input, outputs the composition $\cA(\cX) = (\cA_1(\cX), \cA_2(\cX, \cA_1(\cX)))$. Then $\cA$ satisfies $(\alpha,\eps_1+\eps_2)$-RDP.
\end{lemma}

\begin{lemma}[Strong composition for RDP, v2]\label{lem:comp-rdp-prob}
	For any R\'enyi parameter $\alpha \geq 1$ and any two sequences of random variables $X_1, \dots, X_k$ and $Y_1,\dots,Y_k$,
	\[
	\cD_{\alpha}\left( \Prob_{X_{1:k}} \, \big\| \, \Prob_{Y_{1:k}} \right)
	\leq
	\sum_{i=1}^k \sup_{x_{1:i-1}} \Dalplr{
 \Prob_{X_i | X_{1:i-1} = x_{1:i-1}}}{ \Prob_{Y_i | Y_{1:i-1} = x_{1:i-1}} }.
	\]
\end{lemma}

\subsection{Privacy Amplification by Sampling}\label{ssec:prelim:pabs}

A core technique in the DP literature is Privacy Amplification by Sampling~\citep{Kasiviswanathan11} which quantifies the idea that a private algorithm, run on a small random sample of the input, becomes more private. There are several ways of formalizing this. For our purpose of analyzing $\ISGD$, we must understand how distinguishable a noisy stochastic gradient update is when run on two adjacent datasets. This is precisely captured by the ``Sampled Gaussian Mechanism'', which is a composition of two operations: subsampling and additive Gaussian noise. Below we recall a convenient statement of this in terms of RDP from~\citep{rdp_sgm}. We start with a preliminary definition in one dimension.

\begin{defin}[R\'enyi Divergence of the Sampled Gaussian Mechanism]\label{def:sgm}
	For R\'enyi parameter $\alpha \geq 1$, mixing probability $q \in (0,1)$, and noise parameter $\sig > 0$, define
	\[
		S_{\alpha}(q,\sig) := \cD_{\alpha} \left( \cN(0,\sig^2)  \, \big\| \, (1-q) \cN(0,\sig^2) + q \cN(1,\sig^2) \right).
	\]
\end{defin}

Next, we provide a simple lemma that extends this notion to higher dimensions. In words, the proof simply argues that the worst-case distribution $\mu$ is a Dirac supported on a point of maximal norm, and then reduces the multivariate setting to the univariate setting via rotational invariance. The proof is similar to~\citep[Theorem 4]{rdp_sgm}; for convenience, details are provided in~\cref{app:deferred:sgm-extrema}.

\begin{lemma}[Extrema for R\'enyi Divergence of the Sampled Gaussian Mechanism]\label{lem:sgm-extrema}
	For any R\'enyi parameter $\alpha \geq 1$, mixing probability $q \in (0,1)$, noise level $\sigma > 0$, dimension $d \in \N$, and radius $R > 0$,
	\[
	\sup_{\mu \in \cP(R\mathbb{B}_d)} \cD_{\alpha}\left( \cN(0,\sig^2 I_d) \, \big\| \, (1-q) \cN(0,\sig^2 I_d) + q \left( \cN(0,\sig^2 I_d) \ast \mu\right) \right)
	=
	S_{\alpha}(q,\sig/R),
	\]
	where above $\cP(R\mathbb{B}_d)$ denotes the set of Borel probability distributions that are supported on the ball of radius $R$ in $\R^d$.
\end{lemma}

Finally, we recall the tight bound~\citep[Theorem 11]{rdp_sgm} on these quantities. This bound restricts to R\'enyi parameter at most $\alpha^*(q,\sigma)$, which is defined to be the largest $\alpha$ satisfying 
$\alpha \leq M \sig^2/2 - \log (\sig^2)$ and $\alpha \leq (M^2\sig^2/2 - \log(5\sig^2))/(M + \log(q\alpha) + 1/(2\sig^2))$, where $M = \log(1 + 1/(q(\alpha-1)))$. While we use Lemma~\ref{lem:sgm} to prove the asymptotics in our theoretical results, we emphasize that i) our bounds can be computed numerically for \emph{any} $\alpha \geq 1$, see Remark~\ref{rem:alpha}; and ii) this upper bound does not preclude $\alpha$ from the typical parameter regime of interest, see the discussion in \S\ref{ssec:intro:cont}.

\begin{lemma}[Bound on R\'enyi Divergence of the Sampled Gaussian Mechanism]\label{lem:sgm}
	Consider R\'enyi parameter $\alpha > 1$, mixing probability $q \in (0,1/5)$, and noise level $\sigma \geq 4$. If $\alpha \leq \alpha^*(q,\sigma)$, then
	\[
	S_{\alpha}(q, \sig) \leq  2\alpha q^2 / \sig^2.
	\]
\end{lemma}

\subsection{Privacy Amplification by Iteration}\label{ssec:prelim:pabi}

The Privacy Amplification by Iteration technique of~\citep{pabi} bounds the privacy loss of an iterative algorithm without ``releasing'' the entire sequence of iterates---unlike arguments based on Privacy Amplification by Sampling, c.f., the discussion in \S\ref{ssec:intro:tech}. This technique applies to processes which are generated by a Contractive Noisy Iteration (CNI). We begin by recalling this definition, albeit in a slightly more general form that allows for two differences. The first difference is allowing the contractions to be random; albeit simple, this generalization is critical for analyzing $\ISGD$ because a stochastic gradient update is random. The second difference is that we project each iterate; this generalization is solely for convenience as it simplifies the exposition.\footnote{Since projections are contractive, these processes could be dubbed Contractive Noisy Contractive Processes (CNCI); however, we use Definition~\ref{def:cni} as it more closely mirrors previous usages of CNI. Alternatively, since the composition of contractions is a contraction, the projection can be combined with $\phi_t$ in order to obtain a bona fide CNI; however, this requires defining auxiliary shifted processes which leads to a more complicated analysis (this was done in the original arXiv version v1).}

\begin{defin}[Contractive Noisy Iteration]\label{def:cni}
	Consider a (random) initial state $X_0 \in \R^d$, a sequence of (random) contractions $\phi_t : \R^d \to \R^d$, a sequence of noise distributions $\xi_t$, and a convex set $\cK$. The Projected Contractive Noisy Iteration $\CNI(X_0,\{\phi_t\}, \{\xi_t\}, \cK)$ is the law of the final iterate $X_T$ of the process
	\[
	X_{t+1} = \proj\left[ \phi_{t+1}(X_t) + Z_{t+1} \right],
	\]
	where $Z_{t+1}$ is drawn independently from $\xi_{t+1}$. 
\end{defin}

Privacy Amplification by Iteration is based on two key lemmas. We recall both below in the case of Gaussian noise since this suffices for our purposes. We begin with a preliminary definition. 

\begin{defin}[Shifted R\'enyi Divergence; Definition 8 of~\citep{pabi}]\label{def:shifted-rdp}
	Let $\mu,\nu$ be probability distributions over $\R^d$. For parameters $z \geq 0$ and $\alpha \geq 1$, the shifted R\'enyi divergence is
	\[
	\cD_{\alpha}^{(z)}\big( \mu \;\|\; \nu \big)
	=
	\inf_{\mu' \; : \; W_{\infty}(\mu,\mu') \leq z} \cD_{\alpha}\big( \mu' \; \|\; \nu \big). 
	\]
\end{defin}

(Recall that the $\infty$-Wasserstein metric $\cW_{\infty}(\mu,\mu')$ between two distributions $\mu$ and $\mu'$ is the smallest real number $w$ for which the following holds: there exists a joint distribution $\Prob_{X,X'}$ with first marginal $X \sim \mu$ and second marginal $X' \sim \mu'$, under which $\|X- X'\| \leq w$ almost surely.)

The shift-reduction lemma~\citep[Lemma 20]{pabi} bounds the shifted R\'enyi divergence between two distributions that are convolved with Gaussian noise. 

\begin{lemma}[Shift-reduction lemma]\label{lem:pabi-shift}
	Let $\mu,\nu$ be probability distributions on $\R^d$. For any $a \geq 0$, 
	\[
	\cD_{\alpha}^{(z)} \left( \mu \ast \cN(0,\sig^2 I_d) \; \big\| \; \nu \ast \cN(0,\sig^2 I_d) \right)
	\leq
	\cD_{\alpha}^{(z+a)} \left( \mu \; \big\| \; \nu \right) + \frac{\alpha a^2}{2\sig^2}.
	\]
\end{lemma}

The contraction-reduction lemma~\citep[Lemma 21]{pabi} bounds the shifted R\'enyi divergence between the pushforwards of two distributions through similar contraction maps. Below we state a slight generalization of~\citep[Lemma 21]{pabi} that allows for \emph{random} contraction maps. The proof of this generalization is similar, except that here we exploit the quasi-convexity of the R\'enyi divergence to handle the additional randomness; details in~\cref{app:deferred:pabi-contract}.

\begin{lemma}[Contraction-reduction lemma, for random contractions]\label{lem:pabi-contract}
	Suppose $\phi, \phi'$ are random functions from $\R^d$ to $\R^d$ such that (i) each is a contraction almost surely; and (ii) there exists a coupling of $(\phi,\phi')$ under which
	$\sup_{z} \|\phi(z) - \phi'(z)\| \leq s$ almost surely. Then for any probability distributions $\mu$ and $\mu'$ on $\R^d$, 
	\[
	\cD_{\alpha}^{(z+s)} \left( \phi_{\#} \mu \| \phi_{\#}' \mu' \right)
	\leq
	\cD_{\alpha}^{(z)}\left( \mu \| \mu' \right).
	\]
\end{lemma}

The original\footnote{Strictly speaking, Proposition~\ref{prop:pabi-original} is a generalization of~\citep[Theorem 22]{pabi} since it allows for randomized contractions and projections in the CNI (c.f., Definition~\ref{def:cni}). However, the proof is identical, modulo replacing the original Contraction-Reduction Lemma with its randomized generalization (Lemma~\ref{lem:pabi-contract}) and analyzing the projection step again using the Contraction-Reduction Lemma.} Privacy Amplification by Iteration argument combines these two lemmas to establish the following bound.

\begin{prop}[Original PABI bound]\label{prop:pabi-original}
	Let $X_T$ and $X_T'$ denote the outputs of $\CNI(X_0,\{\phi_t\},\{\xi_t\},\cK)$ and $\CNI(X_0,\{\phi_t'\},\{\xi_t\},\cK)$ where $\xi_t = \cN(0,\sig_t^2 I_d)$. Let $s_t := \sup_x \|\phi_t(x) - \phi_t'(x)\|$, and consider any sequence $a_1, \dots, a_T$ such that $z_t := \sum_{i = 1}^t (s_i - a_i)$ is non-negative for all $t$ and satisfies $z_T = 0$. Then
	\[
	\Dalplr{\Prob_{X_T}}{\Prob_{X_T'}}
	\leq 
	\frac{\alpha}{2}\sum_{t=1}^T \frac{a_t^2}{\sigma_t^2}.
	\]
\end{prop}


\section{Upper bound on privacy}\label{sec:iid}

In this section, we prove the upper bound in Theorem~\ref{thm:iid}. The formal statement of this result is as follows; see \S\ref{ssec:intro:cont} for a discussion of the mild assumptions on $\sigma$ and $\alpha$. 

\begin{theorem}[Privacy upper bound for \ISGD]\label{thm:iid-ub}
	Let $\cK \subset \R^d$ be a convex set of diameter $D$, and consider optimizing convex losses over $\cK$ that are $L$-Lipschitz and $M$-smooth. For any number of iterations $T$, dataset size $n \in \N$, batch size $b \leq n$, stepsize $\eta \leq 2/M$, noise parameter $\sigma > 8\sqrt{2} L/b$, and initialization $\omega_0 \in \cK$, $\ISGD$ satisfies $(\alpha,\eps)$-RDP for $1 < \alpha \leq \alpha^*(\tfrac{b}{n}, \tfrac{b\sig}{2\sqrt{2}L} )$ and
	\begin{align*}
		\eps \lesssim 
		\frac{\alpha L^2}{n^2\sig^2} \min\left\{T, \Tthresh\right\}\,,
	\end{align*}
	where $\Tthresh := \lceil \frac{Dn}{L\eta} \rceil$.
\end{theorem}

Below, in \S\ref{ssec:iid-pabi} we first isolate a simple ingredient in our analysis as it may be of independent interest. Then in \S\ref{ssec:iid-proof} we prove Theorem~\ref{thm:iid-ub}.

\subsection{Privacy Amplification by Iteration bounds that are not vacuous as $T \to \infty$}\label{ssec:iid-pabi}

Recall from the preliminaries section \S\ref{ssec:prelim:pabi} that Privacy Amplification by Iteration arguments, while tight for a small number of iterations $T$, provide vacuous bounds as $T \to \infty$ (c.f., Proposition~\ref{prop:pabi-original}). The following proposition overcomes this by establishing privacy bounds which are \emph{independent} of the number of iterations $T$. 
This result only requires additionally assuming that $\|X_{\tau} - X_{\tau}'\|$ is bounded at some intermediate time $\tau$. This is a mild assumption that is satisfied automatically if, e.g., both CNI processes are in a constraint set of bounded diameter. 
	
\begin{prop}[New PABI bound that is not vacuous as $T \to \infty$]\label{prop:pabi-new}
	Let $X_T$, $X_T'$, and $s_t$ be as in Proposition~\ref{prop:pabi-original}. Consider any $\tau \in \{0, \dots, T-1\}$ and sequence $a_{\tau + 1}, \dots, a_{T}$ such that $z_t := D + \sum_{i=\tau+1}^t (s_i - a_i)$ is non-negative for all $t$ and satisfies $z_T = 0$. If $\cK$ has diameter $D$, then:
	\[
	\Dalplr{\Prob_{X_T}}{\Prob_{X_T'}}
	\leq 
	\frac{\alpha}{2}\sum_{t=\tau + 1}^T \frac{a_t^2}{\sigma_t^2}.
	\]
	\end{prop}
	
	In words, the main idea behind Proposition~\ref{prop:pabi-new} is simply to change the original Privacy Amplification by Iteration argument---which bounds the shifted divergence at iteration $T$, by the shifted divergence at iteration $T-1$, and so on all the way to the shifted divergence at iteration $0$---by instead stopping the induction earlier. Specifically, only unroll to iteration $\tau$, and then use boundedness of the iterates to control the shifted divergence at that intermediate time $\tau$.  
	
 	\par We remark that this new version of PABI uses the shift in the shifted R\'enyi divergence for a different purpose than previous work: rather than just using the shift to bound the bias incurred from updating on two different losses, here we also use the shift to exploit the boundedness of the constraint set.
	
	\begin{proof}[Proof of Proposition~\ref{prop:pabi-new}]
		Bound the divergence at iteration $T$ by the shifted divergence at iteration $T-1$ as follows:
		\begin{align*}
			\Dalplr{ \Prob_{X_T} }{ \Prob_{X_T'} }
			&=
			\Dalpshiftlr{z_T}{ \Prob_{X_T} }{ \Prob_{X_T'} }
			\\ &= 
			\Dalpshiftlr{z_{T-1} + s_T - a_T}{ \Prob_{\proj[\phi_T(X_{T-1}) + Z_T]} }{ \Prob_{\proj[\phi_{T-1}'(X_{T-1}') + Z_T'}] }
			\\ &\leq 	\Dalpshiftlr{z_{T-1} + s_T - a_T}{ \Prob_{\phi_T(X_{T-1}) + Z_T} }{ \Prob_{\phi_{T-1}'(X_{T-1}') + Z_T'}}
			\\ &\leq 
			\Dalpshiftlr{z_{T-1} + s_T}{ \Prob_{\phi_T(X_{T-1})} }{ \Prob_{\phi_{T-1}'(X_{T-1}')} } + \frac{\alpha a_T^2}{2\sigma_T^2} 
			\\ &\leq 
			\Dalpshiftlr{z_{T-1}}{ \Prob_{X_{T-1}} }{ \Prob_{X_{T-1}}' } + \frac{\alpha a_T^2}{2\sigma_T^2}.
		\end{align*}
		Above, the first step is because $z_T = 0$; the second step is by the iterative construction of $X_T, X_T',z_T$; the third and final steps are by the the contraction-reduction lemma (Lemma~\ref{lem:pabi-contract}), and the penulimate step is by the shift-reduction lemma (Lemma~\ref{lem:pabi-shift}).
		\par By repeating the above argument, from $T$ to $T-1$ all the way to $\tau$, we obtain:
		\[
		\Dalplr{ \Prob_{X_T} }{ \Prob_{X_T'} }
		\leq
		\Dalpshiftlr{z_{\tau}}{ \Prob_{X_{\tau}} }{ \Prob_{X_{\tau}}' } + \frac{\alpha }{2} \sum_{t = \tau + 1}^T \frac{a_t^2}{\sigma_t^2}.
		\]
		Now observe that the shifted R\'enyi divergence on the right hand side vanishes because $z_{\tau} = D$.
\end{proof}

\subsection{Proof of Theorem~\ref{thm:iid-ub}}\label{ssec:iid-proof}

\subsubsection*{Step 1: Coupling the iterates}\label{sssec:iid:1-couple}
Suppose $\cX = \{x_1,\dots,x_n\}$ and $\cX' = \{x_1',\dots,x_n'\}$ are adjacent datasets; that is, they agree $x_i = x_i'$ on all indices $i \in [n] \setminus \{i^*\}$ except for at most one index $i^* \in [n]$. Denote the corresponding loss functions by $f_{i}$ and $f_{i}'$, where $f_i = f_{i}'$ except possibly $f_{i^*} \neq f_{i^*}'$. Consider running $\ISGD$ on either dataset $\cX$ or $\cX'$ for $T$ iterations---call the resulting iterates $\{W_t\}_{t=0}^T$ and $\{W_t'\}_{t=0}^T$, respectively---where we start from the same point $w_0 \in \cK$ and couple the sequence of random batches $\{B_t\}_{t=0}^{T-1}$ and the random noise injected in each iteration. That is, 
\begin{align*}
	W_{t+1} &= \proj\left[ W_t - \frac{\eta}{b} \sum_{i \in B_t} \gradf_i(W_t) + Y_t + Z_t \right] \\
	W_{t+1}' &= \proj\left[ W_t' - \frac{\eta}{b} \sum_{i \in B_t} \gradf_i(W_t') + Y_t + Z_t' \right] 
\end{align*}
for all $t \in \{0,\dots,T-1\}$, where we have split the Gaussian noise into terms $Y_t \sim \cN(0,\eta^2\sig_1^2 I_d)$, $Z_t \sim \cN(0,\eta^2 \sig_2^2 I_d)$, $Z_t' \sim \cN(0,\eta^2\sig_2^2 I_d) + \frac{\eta}{b}  \left[ \gradf_{i^*}'(W_t') -  \gradf_{i^*}(W_t') \right] \cdot \mathds{1}_{i^* \in B_t}$, for any\footnote{Simply choosing $\sig_1 = \sig_2 = \sig/\sqrt{2}$ yields our minimax-optimal asymptotics; and we do this at the end of the proof. However, optimizing the constants in this ``noise splitting'' is helpful in practice (details in~\cref{app:numerics}), and as such we leave $\sig_1$ and $\sig_2$ as variables until the end of the proof.} numbers $\sig_1,\sig_2 > 0$ satisfying $\sig_1^2 + \sig_2^2 = \sig^2$. In words, this noise-splitting enables us to use the noise for both the Privacy Amplification by Sampling and Privacy Amplification by Iteration arguments below.
\par Importantly, notice that in the definition of $W_{t+1}'$, the gradient is taken w.r.t. loss functions corresponding to data set $\cX$ rather than $\cX'$; this is then corrected via the bias in the noise term $Z_t'$. Notice also that this bias term in $Z_t'$ is only realized (i.e., $Z_t'$ is possibly non-centered) with probability $1 - b/n$ because the probability that $i^*$ is in a random size-$b$ subset of $[n]$ is
\begin{align}
		\Prob\left[ i^* \in B_t \right] = 
	\frac{b}{n}.
	\label{eq:iid-prob}
\end{align}

\subsubsection*{Step 2: Interpretation as conditional CNI sequences}\label{sssec:iid:2-cni}

Observe that conditional on the event that $Z_t = Z_t'$ are equal (call their value $z_t$), then
\begin{align*}
	W_{t+1} &= \proj\left[ \phi_t(W_t) + Y_t \right] \\
	W_{t+1}' &= \proj\left[ \phi_t(W_t') + Y_t \right]
\end{align*}
where
\begin{align}
	\phi_t(\omega) := \omega - \frac{\eta}{b} \sum_{i \in B_t} \nabla f_i(\omega) + z_t\,.
	\label{eq-pf:isg:phi-t}
\end{align}
Since the following lemma establishes that $\phi_t$ is contractive, we conclude that conditional on the event that $Z_t = Z_t'$ for all $t \geq \tau$, then $\{W_t\}_{t \geq \tau}$ and $\{W_t'\}_{t \geq \tau}$ are projected CNI (c.f., Definition~\ref{def:cni}) with respect to the same update functions. Here, $\tau \in \{0,\dots,T-1\}$ is a parameter that will be chosen shortly. Intuitively, $\tau$ is the horizon for which we bound all previous privacy leakage only through the fact that $W_{\tau}, W_{\tau}'$ are within distance $D$ from each other, see the proof overview in \S\ref{ssec:intro:tech}. 

\begin{obs}\label{obs:iid-contractive}
	The function $\phi_t$ defined in~\eqref{eq-pf:isg:phi-t} is contractive. 
\end{obs}
\begin{proof}
		For any $\omega,\omega'$, 
	\begin{align*}
		\norm { \phi_t(\omega) - \phi_t(\omega') }
		&=
		\norm{ \big( \omega - \frac{\eta}{b}\sum_{i \in B_t} \nabla f_i(\omega) \big)
			- 
			\big( \omega' - \frac{\eta}{b}\sum_{i \in B_t} \nabla f_i(\omega') \big) }\
		\\ &\leq
		\frac{1}{b} \sum_{i \in B_t} \norm{ \big( \omega - \eta \nabla f_i(\omega) \big)
			- 
			\big( \omega' - \eta \nabla f_i(\omega') \big)
		}
		\\ &\leq \frac{1}{b} \sum_{i \in B_t} \norm{ \omega-\omega'}
		\\ &= \norm{ \omega -\omega' }\,.
	\end{align*}
	by plugging in the definition of $\phi_t$, using the triangle inequality, and then using the fact that the stochastic gradient update $\omega \mapsto \omega - \eta \nabla f_i(\omega)$ is a contraction (Lemma~\ref{lem:gd-contraction}).
\end{proof}

\subsubsection*{Step 3: Bounding the privacy loss}\label{sssec:iid:3-bound}
Recall that we seek to upper bound $\Dalp{\Prob_{W_T}}{\Prob_{W_T'}}$. We argue that:
\begin{align}
	\cD_{\alpha}\left(\Prob_{W_T}\; \big\|\; \Prob_{W_T'}\right)
	&\leq 
	\cD_{\alpha}\left( \Prob_{W_T, Z_{\tau:T-1}} \; \big\| \; \Prob_{W_T', Z_{\tau:T-1}'} \right) 
	\nonumber
	\\ &\leq 
	\underbrace{\cD_{\alpha}\left(  \Prob_{Z_{\tau:T-1}} \; \big\| \;  \Prob_{Z_{\tau:T-1}'} \right)}_{\circled{1}}
	+ 
	\underbrace{\sup_z
	\cD_{\alpha}\left( \Prob_{W_{T} | Z_{\tau:T-1} = z} \; \big\| \; \Prob_{ W_T' |  Z_{\tau:T-1}' = z} \right)}_{\circled{2}} 	\label{eq-pf:isgd:rdp}
\end{align}
Above, the first step is by the post-processing inequality for the R\'enyi divergence (Lemma~\ref{lem:process-rdp}), and the second step is by the strong composition rule for the R\'enyi divergence (Lemma~\ref{lem:comp-rdp-prob}). 
\\ \\ \underline{Step 3a: Bounding \circled{1}, using Privacy Amplification by Sampling.}
We argue that
\begin{align}
\circled{1}
&=
\cD_{\alpha}\left( \Prob_{Z_{\tau:T-1}} \; \big\| \; \Prob_{Z_{\tau:T-1}'} \right) \nonumber
\\ &\leq \sum_{t=\tau}^{T-1} \sup_{z_{\tau:t-1}} \cD_{\alpha} \left( \Prob_{Z_t | Z_{\tau:t-1} = z_{\tau:t-1}} \;\big\|\; \Prob_{Z_t' | Z_{\tau:t-1}' = z_{\tau:t-1}} \right) \nonumber
\\ &=
\sum_{t=\tau}^{T-1} \cD_{\alpha} \left( \cN(0,\eta^2 \sig_2^2 I_d) \; \big\| \; (1 - \tfrac{b}{n}) \cN(0,\eta^2 \sig_2^2 I_d) + \tfrac{b}{n} \cN(m_t,\eta^2\sig_2^2 I_d) \right) \nonumber
\\ &\leq (T-\tau) S_{\alpha} \left( \frac{b}{n}, \frac{b\sig_2}{2L} \right).
\label{eq-pf:isgd:A}
\end{align}
Above, the first step is the definition of \circled{1}. The second step is by the strong composition rule for the R\'enyi divergence (Lemma~\ref{lem:comp-rdp-prob}). The third step is because for any $z_{\tau:t-1}$, the law of $Z_t$ conditional on $Z_{\tau:t-1} = z_{\tau:t-1}$ is the Gaussian distribution $\cN(0,\eta^2 \sig_2^2 I_d)$; and by~\eqref{eq:iid-prob}, the law of $Z_t'$ conditional on $Z_{\tau:t-1} = z_{\tau:t-1}$ is the mixture distribution that is $\cN(0,\eta^2 \sig_2^2 I_d)$ with probability $1 - b/n$, and otherwise is $\cN(m_t,\eta^2\sig_2^2 I_d)$ where $m_t :=  \tfrac{\eta}{b}[\gradf_{i^*}(W_t') -\gradf_{i^*}'(W_t')]$. The final step is by the bound in Lemma~\ref{lem:sgm-extrema} on the R\'enyi divergence of the Sampled Gaussian Mechanism, combined with the observation that $\|m_t\| \leq 2 \eta L/b$, which is immediate from the triangle inequality and the $L$-Lipschitz smoothness of the loss functions.
\\ \\ \underline{Step 3b: Bounding \circled{2}, using Privacy Amplification by Iteration.} As argued in step 2, conditional on the event that $Z_t = Z_t'$ for all $t \geq \tau$, then $\{W_t\}_{t \geq \tau}$ and $\{W_t'\}_{t \geq \tau}$ are projected CNI with respect to the same update functions. Note also that $\|W_{\tau} - W_{\tau'}\| \leq D$ since the iterates lie in the constraint set $\cK$ which has diameter $D$. Therefore we may apply the new Privacy Amplification by Iteration bound (Proposition~\ref{prop:pabi-new}) with $s_t \equiv 0$ and $a_t \equiv D/(T-\tau)$ to obtain:
\begin{align}
	\circled{2}
	=
	\sup_{z}  
	\cD_{\alpha}\left( \Prob_{ W_{T} |  Z_{\tau:T-1} = z } \; \big\| \; \Prob_{ W_{T}' |  Z_{\tau:T-1}' = z } \right) 
	\leq
	\frac{\alpha D^2}{2\eta^2 \sig_1^2 (T-\tau)}.
	\label{eq-pf:isgd:B}
\end{align}

\subsubsection*{Step 4: Putting the bounds together}\label{sssec:iid:4-conclude}
 By plugging into~\eqref{eq-pf:isgd:rdp} the bound~\eqref{eq-pf:isgd:A} on \circled{1} and the bound~\eqref{eq-pf:isgd:B} on \circled{2}, we conclude that the algorithm is $(\alpha,\eps)$-RDP for 
\begin{align}
	\eps \leq 
	\min_{\tau \in\{0, \dots, T-1\}}\left\{
	(T-\tau) \, S_{\alpha} \left( \frac{b}{n}, \frac{b\sig_2}{2L}\right) + \frac{\alpha D^2}{2\eta^2 \sig_1^2 (T-\tau)}
	\right\}
	\label{eq-pf:isgd:nonasymp}
\end{align}
By Lemma~\ref{lem:sgm}, $S_{\alpha}( \tfrac{b}{n}, \tfrac{b\sig_2}{2L} ) \leq 8 \alpha (\tfrac{L}{n\sig_2})^2$ for $\alpha \leq \alpha^*(\tfrac{b}{n}, \tfrac{b\sig_2}{2L} )$ and $\sigma_2 \geq 8L/b$.\footnote{While Lemma~\ref{lem:sgm} requires $b < n/5$, the case $b \geq n/5$ has an alternate proof that does not require Lemma~\ref{lem:sgm} (or any of its assumptions). Specifically, replace~\eqref{eq-pf:isgd:A} with the upper bound 
	$
	\cD_{\alpha} ( \cN(0,\eta^2 \sig_2^2 I_d)  \;\| \; \cN(m_t,\eta^2\sig_2^2 I_d) ) = 
	\alpha \|m_t\|^2 / (2 \eta^2 \sig_2^2) = 2\alpha L^2 / b^2 \sig_2^2
	$
	 by using the well-known formula for the R\'enyi divergence between Gaussians. This is tight up to a constant factor, and the rest of the proof proceeds identically.
}
\par By setting $\sig_1 = \sig_2 = \sig/\sqrt{2}$, we have that up to a constant factor, 
\begin{align*}
	\eps
	\lesssim
	\frac{\alpha L^2}{\sig^2}
		\min_{\tau \in \{0,\dots,T-1\}} \left\{
			\frac{(T-\tau)}{n^2} + \frac{D^2}{\eta^2L^2(T-\tau)}
	\right\}.
\end{align*}
\par Bound this minimization by
\[
	\min_{\tau \in \{0,\dots,T-1\}} \left\{
	\frac{(T-\tau)}{n^2} + \frac{D^2}{\eta^2L^2(T-\tau)}
	\right\}
	=
	\min_{R \in \{1,\dots,T\}} \left\{
	\frac{R}{n^2} + \frac{D^2}{\eta^2L^2R}
	\right\}
	\lesssim
	\frac{D}{\eta L n},
\]
where above the first step is by setting $R = T-\tau$, and the second step is by setting $R = \Tthresh = \lceil \tfrac{Dn}{L\eta} \rceil$ (this can be done if
$T \gtrsim \Tthresh$).
Therefore, by combining the above two displays, we obtain
\[
	\eps
	\lesssim 
	\frac{\alpha L^2}{n^2\sig^2}
	\min\left\{
		T, \Tthresh
	\right\}.
\]
Here the first term in the minimization comes from the simple bound~\eqref{eq:simple-bound} which scales linearly in $T$. This completes the proof of Theorem~\ref{thm:iid-ub}.

\section{Lower bound on privacy}\label{sec:lb}

In this section, we prove the lower bound in Theorem~\ref{thm:iid}. This is stated formally below and holds even for linear loss functions in one dimension, in fact even when all but one of the loss functions are zero. See \S\ref{ssec:intro:cont} for a discussion of the mild assumptions on the diameter.

\begin{theorem}[Privacy lower bound for $\ISGD$]\label{thm:iid-lb}
	There exist universal constants\footnote{We prove this for $c_{\sigma} = 10^{-3}$, $c_D = 10^3$, $c_{\alpha} = 10^{-7}$, $\bar{\alpha} = 10^2$; no attempt has been made to optimize these constants.\label{fn:lb-constants}} $c_{D}$, $c_{\sigma}$, $c_{\alpha}$, $\bar{\alpha}$ and a family of $L$-Lipschitz linear loss functions over the interval $\cK = [-D/2,D/2] \subset \R$ such that the following holds. Consider running $\ISGD$ from arbitrary initialization $\omega_0$ with any parameters satisfying $D \geq c_D \eta L$ and $\sig^2 \leq c_{\sigma} D^2/(\eta^2 \Tthresh)$. Then $\ISGD$ is not $(\bar{\alpha},\eps)$-RDP for
	\begin{align}
		\eps \leq c_{\alpha } \frac{\bar{\alpha} L^2}{n^2\sig^2} \min\left\{T, \Tthresh\right\},
		\label{eq:thm-iid-lb}
	\end{align}
	where $\Tthresh := 0.75 \tfrac{Dn}{L\eta}$.
\end{theorem}

\paragraph*{Proof sketch of Theorem~\ref{thm:iid-lb}.} (Full details in~\cref{app:lb}.)

\par \underline{Construction.} Consider datasets $\cX = \{x_1, \dots, x_{n-1}, x_n\}$ and $\cX' = \{x_1, \dots, x_{n-1}, x_n'\}$ which differ only on $x_n'$, and corresponding functions which are all zero $f_1(\cdot) = \dots = f_n(\cdot) = 0$, except for $f_{n}'(\omega) = L(D-\omega)$. Clearly these functions are linear and $L$-Lipschitz. The intuition behind this construction is that running $\ISGD$ on $\cX$ or $\cX'$ generates a random walk that is clamped to stay within the interval $\cK$---but with the key difference that running $\ISGD$ on dataset $\cX$ generates a \emph{symmetric} random walk $\{W_t\}$, whereas running $\ISGD$ on dataset $\cX'$ generates a \emph{biased} random walk $\{W_t'\}$ that biases right with probability $b/n$ each step. That is,
\begin{align*}
	W_{t+1} = \Pi_{\cK} \Big[ W_t + Z_t \Big]
\qquad \text{and} \qquad 
	W_{t+1}' = \Pi_{\cK} \Big[ W_t' + Y_t + Z_t \Big]
\end{align*}
where the processes are initialized at $\omega_0 = \omega_0' = 0$, each random increment $Z_t \sim \cN(0,\eta^2\sig^2)$ is an independent Gaussian, and each bias $Y_t$ is $\eta L/b$ with probability $b/n$ and otherwise is $0$. 

\par \underline{Key obstacle.} The high-level intuition behind this construction is simple to state: the bias of the random walk $\{W_t'\}$ makes it distinguishable (to the minimax-optimal extent, as we show) from the symmetric random walk $\{W_t\}$. However, making this intution precise is challenging because the distributions of the iterates $W_t, W_t'$ are intractable to reason about explicitly---due to the highly non-linear interactions between the projections and the random increments. Thus we must establish the distinguishability of $W_T, W_T'$ without reasoning explicitly about their distributions.

\par \underline{Key technical ideas.} A first, simple observation is that it suffices to prove Theorem~\ref{thm:iid-lb} in the constant RDP regime of $T \geq \Tthresh$, since the linear RDP regime of $T \leq \Tthresh$ then follows by the strong composition rule for RDP (Lemma~\ref{lem:comp-rdp-alg}), as described in Appendix~\ref{app:lb}. Thus it suffices to show that the final iterates $W_T, W_T'$ are distinguishable after $T \geq \Tthresh$ iterations.

\par A natural attempt to distinguish $W_T, W_T'$ for large $T$ is to test positivity. This is intuitively plausible because $\Prob[W_T \geq 0] = 1/2$ by symmetry, whereas we might expect $\Prob[W_T' \geq 0] \gg 1/2$ since the bias of $W_T'$ pushes it to the top half of the interval $\cK = [-D/2,D/2]$. 
Such a discrepancy would establish an ($\eps,\del$)-DP bound that, by the standard RDP-to-DP-conversion in Lemma~\ref{lem:rdp-to-dp}, would imply the desired ($\alpha,\eps$)-RDP bound in Theorem~\ref{thm:iid-lb}. However, the issue is how to prove the latter statement $\Prob[W_T' \geq 0] \gg 1/2$ without an explicit handle on the distribution of $W_T'$.

\par To this end, the key technical insight is to define an auxiliary process $W_t''$ which intializes $\Tthresh$ iterations before the final iteration $T$ at the lowest point in $\cK$, namely $W_{T - \Tthresh}'' = -D/2$, and then updates in an analogously biased way as the $W_t'$ process except without projections at the bottom of $\cK$. That is,
\begin{align*}
	W_{t+1}'' := \min\left( W_t'' + Y_t + Z_t, D/2 \right).
\end{align*}
The point is that on one hand, $W_t'$ stochastically dominates $W_t''$, so that it suffices to show $\Prob[W_T'' \geq 0] \gg 1/2$. And on the other hand, $W_t''$ is easy to analyze because, as we show, with overwhelming probability no projections occur. This lack of projections means that, modulo a probability $\delta$ event which is irrelevant for the purpose of $(\eps,\del)$-DP bounds, $W_T''$ is the sum of (biased) independent Gaussian increments---hence it has a simple explicitly computable distribution: it is Gaussian.

\par It remains to establish that (i) with high probability no projections occur in the $W_t''$ process, so that the aforementioned Gaussian approximates the law of $W_T''$, and (ii) this Gaussian is positive with probability $\gg 1/2$, so that we may conclude the desired DP lower bound. Briefly, the former amounts to bounding the hitting time of a Gaussian random walk, which is a routine application of martingale concentration. And the latter amounts to computing the parameters of the Gaussian. A back-of-the-envelope calculation shows that the total bias in the $W_t''$ process is $\sum_{t = T - \Tthresh}^{T-1} Y_t \approx \Tthresh \eta L / n = 3D/4$, and conditional on such an event, the Gaussian approximating $W_T''$ has mean roughly $-D/2 + 3D/4 = D/4$ and variance $\Tthresh \eta^2 \sig^2 \leq D^2/1000$, and therefore is positive with probability $\gg 1/2$, as desired. Full proof details are provided in~\cref{app:lb}.


\section{Extensions}\label{app:extensions}

This section provides details for the extensions mentioned in \S\ref{ssec:intro:cont}. Specifically, we show how our techniques readily extend to strongly convex losses (\S\ref{app:sc}), to cyclic batch updates (\S\ref{app:cyclic}), to non-uniform stepsizes (\S\ref{app:nonuniform}), and to regularized losses (\S\ref{app:reg}). Our techniques readily extend to any combinations of these settings, e.g. $\ISGD$ with cyclic batches and non-uniform stepsizes on strongly convex losses; the details are straightforward and omitted for brevity.

\subsection{Strongly convex losses}\label{app:sc}

Here we show a significantly improved privacy loss if the functions are $m$-strongly convex. In particular, we show that after only $\Otilde(1/(\eta m))$ iterations, there is no further privacy loss. Throughout, the notation $\Otilde$ suppresses logarithmic factors in the relevant parameters.

\begin{theorem}[Privacy upper bound for \ISGD, in strongly convex setting]\label{thm:iid-sc}
Let $\cK \subset \R^d$ be a convex set of diameter $D$, and consider optimizing losses over $\cK$ that are $L$-Lipschitz, $m$-strongly convex, and $M$-smooth. For any number of iterations $T$, dataset size $n \in \N$, batch size $b \leq n$, noise parameter $\sigma > 8\sqrt{2}L/b$, and initialization $\omega_0 \in \cK$, $\ISGD$ with stepsize $\eta < 2/M$ satisfies $(\alpha,\eps)$-RDP for $1 < \alpha < \alpha^*( \tfrac{b}{n}, \tfrac{b\sig}{2\sqrt{2}L})$ and
\begin{align}
	\eps \lesssim 
	\frac{\alpha L^2}{n^2\sig^2} \cdot \min\big\{T,
	\Tthresh\big\}\,,
	\label{eq:thm-iid-sc}
\end{align}
where $\Tthresh = \Otilde(1/\log(1/c))$ and $c := \max_{\lambda \in \{m,M\}}|1 - \eta \lambda| < 1$.
\end{theorem}

We make two remarks about this result.

\begin{remark}[Intepretation of $c$ and $\Tthresh$]
	The operational interpretation of $c$ is that it is the contraction coefficient for a gradient descent iteration with stepsize $\eta$ (c.f., Lemma~\ref{lem:gd-contraction-sc} below). In the common setting of $\eta \leq 1/M$, then $c = 1 - \eta m \leq \exp(-\eta m)$, whereby
	\[
		\Tthresh = \Otilde\left( \frac{1}{\log 1/c} \right) = \Otilde\left( \frac{1}{\eta m} \right)\,.
	\]
	Another important example is $\eta = 2/(M+m)$, i.e., when the stepsize is optimized to minimize the contraction coefficient $c$. In this case, $c = (\kappa-1)/(\kappa+1) \leq e^{-1/\kappa}$ where $\kappa := M/m \geq 1$ denotes the condition number, and thus $\Tthresh = \Otilde(\kappa)$.
\end{remark}

\begin{remark}[Bounded diameter is unnecessary for convergent privacy in the strongly convex setting]
	Unlike the convex setting, in this strongly convex setting, the bounded diameter assumption can be removed. Specifically, for the purposes of DP (rather than RDP), the logarithmic dependence of $\Tthresh$ on $D$ (hidden in the $\Otilde$ above) can be replaced by logarithmic dependence on $T$, $\eta$, $L$, and $\sig$ since the SGD trajectories move from the initialization point by $O(T\eta ( L + \sigma))$ with high probability, and so this can act as the ``effective diameter''. 
	The intuition behind why a diameter assumption is needed for the convex setting but not for the strongly convex setting is that without strong convexity, $\SGD$ is only weakly contractive, which means that its instability increases with the number of iterations; whereas in the strongly convex setting, $\SGD$ is strongly contractive, which means that movements from far enough in the past are effectively ``erased''---a pre-requisite for convergent privacy.	
\end{remark}

The proof of Theorem~\ref{thm:iid-sc} (the setting of strongly convex losses) is similar to the proof of Theorem~\ref{thm:iid-ub} (the setting of convex losses), except for two key differences:
\begin{itemize}
	\item[(i)] Strong convexity of the losses ensures that the update function $\phi_t$ in the Contractive Noisy Iterations is \emph{strongly} contractive. (Formalized in Observation~\ref{obs:iid-contractive-sc}.)
	\item[(ii)] This strong contractivity of the update function ensures exponentially better bounds in the Privacy Amplification by Iteration argument. (Formalized in Proposition~\ref{prop:pabi-new-sc}.)
\end{itemize} 

We first formalize the change (i).

\begin{obs}[Analog of Observation~\ref{obs:iid-contractive} for strongly convex losses]\label{obs:iid-contractive-sc}
	For all $t$, the function $\phi_t$ defined in~\eqref{eq-pf:isg:phi-t} is almost surely $c$-contractive, where $c$ is the quantity in Theorem~\ref{thm:iid-sc}.
\end{obs}
\begin{proof}
	Identical to the proof of Observation~\ref{obs:iid-contractive}, except use the fact that a gradient step is not simply contractive (Lemma~\ref{lem:gd-contraction}), but in fact strongly contractive when the function is strongly convex (this fact is recalled in Lemma~\ref{lem:gd-contraction-sc} below; see, e.g.,~\citep[Theorem 3.12]{bubeck2015convex} for a proof).
\end{proof}

\begin{lemma}[Analog of Lemma~\ref{lem:gd-contraction} for strongly convex losses]\label{lem:gd-contraction-sc}
	Suppose $f : \R^d \to \R$ is an $m$-strongly convex, $M$-smooth function. For any stepsize $\eta < 2/M$, the mapping $\omega \mapsto \omega - \eta \nabla f(\omega)$ is $c$-contractive, for $c := \max_{\lambda \in \{m,M\}}|1 - \eta \lambda| < 1$.
\end{lemma}

Next we formalize the change (ii). 

\begin{prop}[Analog of Proposition~\ref{prop:pabi-new} for strongly convex losses]\label{prop:pabi-new-sc}
	Consider the setup in Proposition~\ref{prop:pabi-new}, and additonally assume that $\phi_t,\phi_t'$ are almost surely $c$-contractive. Consider any $\tau \in \{0, \dots, T-1\}$ and any reals $a_{\tau + 1}, \dots, a_{T}$ such that $z_t := c^{t-\tau} D + \sum_{i=\tau+1}^t c^{t - i}(s_i - a_i)$ is non-negative for all $t$ and satisfies $z_T = 0$. Then:
	\[
	\Dalplr{\Prob_{X_T}}{\Prob_{X_T'}}
	\leq 
	\frac{\alpha}{2}\sum_{t=\tau + 1}^T \frac{a_t^2}{\sigma_t^2}.
	\]
\end{prop}
\begin{proof}
	Identical to the proof of Proposition~\ref{prop:pabi-new}, except use the contraction-reduction lemma for strong contractions (Lemma~\ref{lem:pabi-contract-sc} below rather than Lemma~\ref{lem:pabi-contract}), and the inductive relation $z_{t+1} = cz_t + s_{t+1} - a_{t+1}$.
\end{proof}

\begin{lemma}[Analog of Lemma~\ref{lem:pabi-contract} for strongly convex losses]\label{lem:pabi-contract-sc}
	Consider the setup of Lemma~\ref{lem:pabi-contract-sc}, except additionally assuming that $\phi,\phi'$ are each $c$-contractive almost surely. Then
	\[
	\cD_{\alpha}^{(cz+s)} \left( \phi_{\#} \mu \| \phi_{\#}' \mu' \right)
	\leq
	\cD_{\alpha}^{(z)}\left( \mu \| \mu' \right).
	\]
\end{lemma}
\begin{proof}
	Identical to the proof of Lemma~\ref{lem:pabi-contract}, except use the fact that $\phi$ is almost surely a $c$-contraction to bound the second term in~\eqref{eq-pf:contraction-iid}, namely by $W_{\infty}( \phi_{\#}\nu, \phi_{\#}\mu ) \leq c W_{\infty}( \nu,\mu) \leq cz$.
\end{proof}

We now combine the changes (i) and (ii) to prove Theorem~\ref{thm:iid-sc}.

\begin{proof}[Proof of Theorem~\ref{thm:iid-sc}]
	We record the differences to Steps 1-4 of the proof in the convex setting (Theorem~\ref{thm:iid-ub}). Step 1 is identical for coupling the iterates. Step 2 is identical for constructing the conditional CNI, except that in this strongly convex setting, the update functions $\phi_t$ are not simply contractive (Observation~\ref{obs:iid-contractive}), but in fact $c$-contractive (Observation~\ref{obs:iid-contractive-sc}). We use this to establish an improved bound on the term \circled{2} in Step 3. Specifically, invoke Proposition~\ref{prop:pabi-new-sc} with $s_t = 0$ for all $t \in \{\tau + 1, \dots, T\}$, $a_t = 0$ for all $t \in \{\tau + 1, \dots, T-1\}$ and $a_{T} = c^{T - \tau}D$ to obtain
	\begin{align}
		\circled{2}
		=
		\sup_{z}  
		\cD_{\alpha}\left( \Prob_{ W_T |  Z_{\tau:T-1} = z } \; \big\| \; \Prob_{ W_T' |  Z_{\tau:T-1}' = z } \right) 
		\leq
		c^{2(T -  \tau)}
		\frac{\alpha D^2 }{2\eta^2 \sig_1^2} .
		\label{eq-pf:isgd:B-sc}
	\end{align}
	This allows us to conclude the following analog of~\eqref{eq-pf:isgd:nonasymp} for the present strongly convex setting:
	\begin{align}
		\eps \leq 
		\min_{\tau \in \{0, \dots, T-1\}}
		(T-\tau) Q + c^{2(T -  \tau )} \frac{\alpha D^2}{2\eta^2 \sig_1^2 },
		\label{eq-pf:isgd:nonasymp:sc}
	\end{align}
	where $Q := S_{\alpha} ( \tfrac{b}{n}, \tfrac{b\sig_2}{2L})$. Simplifying this in Step 4 requires the following modifications since the asymptotics are different in the present strongly convex setting. Specifically, bound the above by
	\begin{align*}
		\eps
		\lesssim
		\frac{Q}{\log 1/c} \log \left( \frac{\alpha D^2 \log 1/c}{\eta^2 \sig_1^2 Q} \right).
	\end{align*}
	 by setting $\tau = T - \Theta( \log ( (\alpha D^2 \log1/c)/(\eta^2 \sig_1^2 Q))/(\log 1/c)) = T - \tilde{\Theta}(1/(\log1/c))$,
	valid if $\tau \geq 0$. 
	By setting $\sig_1 = \sig_2 = \sig/\sqrt{2}$, and using the bound on $Q$ in the proof of Theorem~\ref{thm:iid-ub}, we conclude the desired bound
	\[
	\eps
	\lesssim 
	\frac{\alpha L^2}{n^2\sig^2}
	\min\left\{
	T, \frac{1}{\log 1/c} \log \left( \frac{\alpha D^2 \log 1/c}{\eta^2 \sig_1^2 Q} \right)
	\right\}
	=
	\frac{\alpha L^2}{n^2\sig^2} \cdot \min\big\{T,\Otilde\left(\frac{1}{\log 1/c}\right)\big\}.
	\]
\end{proof}

\subsection{Cyclic batch updates}\label{app:cyclic}

Here we show how our results can be extended to $\ISGD$ if the batches are chosen cyclically in round-robin fashion rather than randomly. We call this algorithm $\CSGD$. The previous best privacy bound for this algorithm is $(\alpha,\eps)$-RDP upper bound with
\begin{align}
	\eps \lesssim \frac{\alpha L^2}{\sig^2}\left(\frac{1}{b^2}+\frac{T}{n^2}\right),
	\label{eq:pabi:cyclic-without-diam}
\end{align}
which can be obtained from the standard Privacy Amplification by Iteration argument if one extends the proof of~\citep[Theorem 35]{pabi} to arbitrary numbers of iterations $T$ and batch sizes $b$. The key difference in our privacy bound, below, is that it does not increase ad infinitum as $T \to \infty$. This is enabled by the new diameter-aware Privacy Amplification by Iteration bound (Proposition~\ref{prop:pabi-new}).

\begin{theorem}[Privacy upper bound for $\CSGD$]\label{thm:diam-rdp}
	Let $\cK \subset \R^d$ be a convex set of diameter $D$, and consider optimizing convex losses over $\cK$ that are $L$-Lipschitz and $M$-smooth. For any number of iterations $T \in \N$, stepsize $\eta \leq 2/\beta$, noise parameter $\sigma > 0$, initialization $\omega_0 \in \cK$, dataset size $n \in \N$, and batch size $b$, $\CSGD$ satisfies $(\alpha,\eps)$-RDP for any $\alpha > 1$ and
	\begin{align}
		\eps \lesssim \frac{\alpha L^2}{\sig^2} \left( \frac{1}{b^2} + \frac{1}{n^2} \, \min\left\{T, \Tthresh \right\}\right)\,,
	\end{align}
	where $\Tthresh := \lceil \frac{Dn}{L\eta} \rceil$.
\end{theorem}

\par Our analysis of $\CSGD$ is simpler than that of $\ISGD$: because $\CSGD$ updates in a deterministic order, its analysis requires neither Privacy Amplification by Sampling arguments nor the conditionality of the CNI argument in the proof of Theorem~\ref{thm:iid-ub}. This also obviates any assumptions on the R\'enyi parameter $\alpha$ and noise $\sigma$ because the Privacy Amplification by Sampling Lemma~\ref{lem:sgm} is not needed here. On the other hand, despite $\CSGD$ being easier to analyze, this algorithm yields substantially weaker privacy guarantees for small numbers of iterations $T$.

\begin{proof}[Proof of Theorem~\ref{thm:diam-rdp}]
	We prove Theorem~\ref{thm:diam-rdp} in the case $T > \Tthresh$, because as mentioned above, the case $T \leq \Tthresh$ follows from~\eqref{eq:pabi:cyclic-without-diam}, which is a straightforward extension of~\citep[Theorem 35]{pabi}. 
	
	Suppose $\cX = \{x_1,\dots,x_n\}$ and $\cX' = \{x_1',\dots,x_n'\}$ are adjacent datasets; that is, they agree $x_i = x_i'$ on all indices $i \in [n] \setminus \{i^*\}$ except for at most one index $i^* \in [n]$. Denote the corresponding loss functions by $f_{i}$ and $f_{i'}$, where $f_i = f_{i}'$ except possibly $f_{i^*} \neq f_{i^*}'$. Consider running $\CSGD$ on either dataset $\cX$ or $\cX'$ for $T$ iterations---call the resulting iterates $\{W_t\}_{t=0}^T$ and $\{W_t'\}_{t=0}^T$, respectively---where we start from the same point $\omega_0 \in \cK$ and couple the sequence of random batches $\{B_t\}_{t=0}^{T-1}$ and the random noise injected in each iteration. Then
	\begin{align*}
		W_t &= \proj\left[ \phi_t(W_t) + Z_t \right] \\
		W_t' &= \proj\left[ \phi_t'(W_t') + Z_t \right]
	\end{align*}
	where $Z_t \sim \cN(0,\sig^2 I_d)$ and 
	\begin{align*}
		\phi_t(\omega) &:= \omega -
		\frac{\eta}{b}\sum_{i \in B_t} \nabla f_i \left(\omega \right) \\
		\phi_t'(\omega') &:= \omega' - 
		\frac{\eta}{b}\sum_{i \in B_t} \nabla f_i' \left(\omega' \right)
	\end{align*}
	Observe that $\phi_t,\phi_t'$ are contractions since a stochastic gradient update on a smooth function is contractive (by Observation~\ref{obs:iid-contractive} and the triangle inequality). Therefore the SGD iterates are generated by Projected Contractive Noisy Iterations $W_T \sim \CNI(\omega_0,\{\phi_t\},\{\xi_t\},\cK)$ and $W_T' \sim \CNI(\omega_0,\{\phi_t'\},\{\xi_t\},\cK)$ where the noise distributions are $\xi_t = \cN(0,\eta^2 \sigma^2 I_d)$.
	\par 	Therefore we may invoke the new diameter-aware Privacy Amplification by Iteration bound (Proposition~\ref{prop:pabi-new}) using the diameter bound $\tilde{D} := D + \frac{2\eta L}{b}$. To this end, we now set the parameters $s_t$, $a_t$, and $\tau$ in that proposition. Bound
	$
	s_t
	:=
	\sup_{\omega} \|\phi_t(\omega) - \phi_t'(\omega)\|
	\leq \frac{2\eta L}{b}\, \mathds{1}[i^* \in B_t]
	$
	by the $L$-Lipschitz assumption on the losses. Set $\tau = T-\Tthresh$ and assume for simplicity of exposition that the batch size $b$ divides the dataset size $n$, and also that $n/b$ divides $T,\Tthresh$ so that each datapoint is updated the same number of times in $T,\Tthresh$ iterations; the general case is a straightforward extension that just requires additional floors/ceilings. Then set $a_{\tau + t} = \tilde{D}/\Tthresh$ for the first $\lceil i^*/b\rceil - 1$ batches before updating the different datapoint $i^*$, followed by $a_{\tau + t} = \tilde{D}/\Tthresh + 2 \eta L /n$ for the next $\Tthresh - n/b$ batches before the final update on $i^*$, and then $a_{\tau+ t} = \tilde{D}/\Tthresh + (2\eta L)/(b(n/b - \lceil i^*/b \rceil + 1))$ until the final iteration $T$. Observe that then $z_{t} = \tilde{D} + \sum_{i=\tau+1}^t (s_i - a_i)$ satisfies $z_t \geq 0$ for all $t$ and has final value $z_T = 0$, which lets us invoke Proposition~\ref{prop:pabi-new}. Therefore
	\begin{align}
		\Dalplr{\Prob_{W_T}}{\Prob_{W_T'}} 
		\leq
		\frac{\alpha}{2 \eta^2 \sig^2} \sum_{t=\tau + 1}^{T} a_t^2.
		\label{eqn:cyclic-ub-at-sum}
	\end{align}
	\par To simplify, plug in the definition of $a_t$, use the elementary bound $(x+y)^2 \leq 2(x^2+y^2)$, and collect terms to obtain
	\begin{align*}
		\sum_{t=\tau+ 1}^{T} a_t^2
		&= 
		\left( \left \lceil \frac{i^*}{b} \right\rceil - 1 \right)\, \left(\frac{\tilde{D}}{\Tthresh}\right)^2
		+ \left( \Tthresh - \frac{n}{b} \right) \left(\frac{\tilde{D}}{\Tthresh} + \frac{2\eta L}{n}\right)^2
		+ 
		\left(
		\frac{n}{b} - \left \lceil \frac{i^*}{b} \right \rceil + 1
		\right)
		\, \left(\frac{\tilde{D}}{\Tthresh} + \frac{2\eta L}
		{b( \tfrac{n}{b} -  \lceil \tfrac{i^*}{b}  \rceil + 1 )}
		\right)^2
		\\ &\lesssim
		\frac{\tilde{D}^2}{\Tthresh} + \eta^2L^2 \left( 
		\frac{1}{b^2} + \frac{\Tthresh}{n^2}
		\right)
		\\ &\asymp \eta^2 L^2 \left(
		\frac{1}{b^2} + 
		\frac{\Tthresh}{n^2}
		\right).
	\end{align*}
	Combining the above two displays yields
	\begin{align*}
		\Dalplr{\Prob_{W_T}}{\Prob_{W_T'}} 
		\lesssim
		\frac{\alpha L^2}{ \sig^2} \left( 
		\frac{1}{b^2} + \frac{\Tthresh}{n^2}
		\right).
	\end{align*}
	This proves Theorem~\ref{thm:diam-rdp} in the case $T > \Tthresh$, as desired.
\end{proof}

\subsection{Non-uniform stepsizes}\label{app:nonuniform}

The analysis of $\ISGD$ in Theorem~\ref{thm:iid} readily extends to non-uniform stepsize schedules $\{\eta_t\}_{t=0}^{T-1}$. Assume $\eta_t \leq 2/M$ for all $t$ so that the gradient update is contractive (Lemma~\ref{lem:gd-contraction}). Then the analysis in \S\ref{sec:iid} goes through unchanged, and can be improved by choosing non-uniform $a_t$ in the Privacy Amplification by Iteration argument to bound $\circled{2}$. Specifically, since the bound in Proposition~\ref{prop:pabi-new} scales as 
\[
\circled{2} \leq \frac{\alpha}{2\sig_1^2} \sum_{t = \tau + 1}^{T-1} \frac{a_t^2}{\eta_t^2},
\]
the proper choice of $a_t$ is no longer the uniform choice $a_t \equiv D/(T-\tau - 1)$, but instead the non-uniform choice $a_t = D\eta_t / \sum_{t=\tau+1}^{T-1}\eta_t$. This yields the following generalization of the privacy bound~\eqref{eq-pf:isgd:nonasymp}: $\ISGD$ satisfies $(\alpha,\eps)$-RDP for
\begin{align}
	\eps \leq 
	\min_{\tau \in \{0, \dots, T-1\}}
	(T-\tau) Q + (T - \tau - 1)\frac{\alpha D^2}{2 \sig_1^2 (\sum_{t=\tau+1}^{T-1} \eta_t )^2 }.
	\label{eq-pf:isgd:nonasymp-nonunif}
\end{align}
where $Q := S_{\alpha} ( \tfrac{b}{n}, \tfrac{b\sig_2}{2L} )$.

\par Further simplifying this RDP bound~\eqref{eq-pf:isgd:nonasymp-nonunif} requires optimizing over $\tau$, and this optimization depends on the decay rate of the stepsize schedule $\{\eta_t\}$. As an illustrative example, we demonstrate this optimization for the important special case of $\eta_t \asymp t^{-c}$ for $c \in [0,1)$, which generalizes the constant step-size result in Theorem~\ref{thm:iid-ub}.

\begin{example}[Privacy of $\ISGD$ with polynomially decaying stepsize schedule]
	Consider stepsize schedule $\eta_t \asymp t^{-c}$ for $c \in [0,1)$. Then up to a constant factor, the privacy bound~\eqref{eq-pf:isgd:nonasymp-nonunif} gives
	\[
		\eps \lesssim \alpha \frac{ D L T^c }{n \sigma^2} 
	\]
	by a similar optimization over $\tau$ as in the proof of Theorem~\ref{thm:iid-ub}, with the only differences here being that we take $\tau = T - R  - 1$ where 
	$R \asymp DT^c/\sigma \sqrt{ \alpha / Q}$
	 and note that $\sum_{t = \tau+ 1}^{T-1} \eta_t \asymp R T^{-c}$ since $R \ll T$. We therefore obtain the final privacy bound of
	\[
	\eps
	\lesssim 
	\frac{\alpha L^2}{n^2\sig^2}
	\min\left\{
	T, \frac{Dn}{L} T^c
	\right\},
	\]
	where the first term in the minimization is from the simple bound~\eqref{eq:simple-bound} which scales linearly in $T$. 
\end{example}

\subsection{Regularization}\label{app:reg}

In order to improve generalization error, a common approach for training machine learning models---whether differentially private or not---is to regularize the loss functions. Specifically, solve $$\min_{\omega \in \cK} \frac{1}{n}\sum_{i=1}^n \Big( f_i(\omega) + R(\omega) \Big)\,$$ for some convex regularization function $R$. In order to apply existing optimization and privacy analyses for the regularized problem, one could simply use the relevant parameters (Lipschitzness, smoothness, strong convexity) of the regularized losses $F_i := f_i + R$ rather than of $f_i$. 
A standard issue is that the Lipschitzness constant degrades after regularization; that is, typically the Lipschitz constant of $F_i$ can only be bounded by the Lipschitz constant of $f_i$ plus that of $R$. 
\par The purpose of this brief subsection is to point out that when using our privacy analysis, the Lipschitz constant does \emph{not} degrade after regularization. Specifically, although applying our privacy bounds requires using the strong convexity and smoothness constants of the regularized function $F_i$ (as is standard in all existing privacy/optimization analyses), one can use the Lipschitz constant of $f_i$ rather than that of $F_i$. This is because our analysis only uses Lipschitzness in Step 3a to bound the ``gradient sensitivity'' $\sup_{\omega}\|\nabla f_i(\omega) - \nabla f_i'(\omega)\|$---aka the amount that gradients can differ when using losses from adjacent datasets---and this quantity is invariant under regularization because of course $\nabla F_i(\omega) - \nabla F_i'(\omega) = \nabla (f_i + R)(\omega) - \nabla (f_i' + R)(\omega)  = \nabla f_i(\omega) - \nabla f_i'(\omega)$. 

\section{Discussion}\label{sec:discussion}

The results of this paper suggest several natural directions for future work:

\begin{itemize}
	\item \underline{Clipped gradients?} In practical settings, $\ISGD$ implementations sometimes ``clip'' gradients to force their norms to be small, see e.g.,~\citep{abadi2016deep}. 
	In the case of generalized linear models, the clipped gradients can be viewed as gradients of an auxiliary convex loss~\citep{song2021evading}, in which case our results can be applied directly. However, in general, clipped gradients do not correspond to gradients of a convex loss, in which case our results (as well as all other works in the literature that aim at proving convergent privacy bounds) do not apply. Can this be remedied? 
	\item \underline{Average iterate?} Can similar privacy guarantees be established for the average iterate rather than the last iterate? There are fundamental difficulties with trying to proving this: indeed, the average iterate is provably not as private for $\CSGD$~\citep{BassilyFeTaTh19}.
	\item \underline{Adaptive stepsizes?} Can similar privacy guarantees be established for optimization algorithms with adaptive stepsizes? The main technical obstacle is the impact of the early gradients on the trajectory through the adaptivity, and one would need to control the privacy loss accumulation due to this. This appears to preclude using our analysis techniques, at least in their current form. 
	\item \underline{Beyond convexity?} Convergent privacy bounds break down without convexity. This precludes applicability to deep neural networks. Is there any hope of establishing similar results under some sort of mild non-convexity? Due to simple non-convex counterexamples where the privacy of $\ISGD$ diverges, any such extension would have to make additional structural assumptions on the non-convexity (and also possibly change the $\ISGD$ algorithm), although it is unclear how this would even look. Moreover, this appears to require significant new machinery as our techniques are the only known way to solve the convex problem, and they break down in the non-convex setting (see also the discussion in \S\ref{ssec:intro:cont}).
	\item \underline{General techniques?} Can the analysis techniques developed in this paper be used in other settings?  Our techniques readily generalize to any iterative algorithm which interleaves contractive steps and noise convolutions. Such algorithms are common in differentially private optimization, and it would be interesting to apply them to variants of $\ISGD$. 
\end{itemize}

	\paragraph*{Acknowledgements.} We are grateful to Hristo Paskov for many insightful conversations.
	
	\small
	\appendix

\section{Deferred proofs}\label{app:deferred}

\subsection{Proof of~\cref{lem:comp-rdp-prob}}\label{app:deferred:comp-rdp}

	For $\alpha > 1$,
	\begin{align*}
		\exp\Big( (\alpha - 1) \, 
		\Dalplr{ \Prob_{X_{1:k}} }{ \Prob_{Y_{1:k}} } \Big)
		&= 
		\int_{\cX_1 \times \cdots \times \cX_k} \left[ \frac{  \Prob_{X_{1:k}}(x_{1:k}) }{ \Prob_{Y_{1:k}}(x_{1:k}) } \right]^{\alpha}
		d\Prob_{Y_{1:k}}(x_{1:k})
		\\ &=
		\int_{\cX_1 \times \cdots \times \cX_k} \left[ \prod_{i=1}^k \frac{  \Prob_{X_i | X_{1:i-1}=x_{1:i-1}}(x_i) }{ \Prob_{Y_i | Y_{1:i-1}=x_{1:i-1}}(x_i)  } \right]^{\alpha}
		\prod_{i=1}^k d \Prob_{Y_i | Y_{1:i-1}=x_{1:i-1}}(x_i)
		\\ &\leq 
		\prod_{i=1}^k \sup_{x_1 \in \cX_1, \dots, x_{i-1}\in \cX_{i-1}} \int_{\cX_i} 
		\left[ \frac{  \Prob_{X_i | X_{1:i-1}=x_{1:i-1}}(x_i) }{ \Prob_{Y_i | Y_{1:i-1}=x_{1:i-1}}(x_i)  } \right]^{\alpha}
		d \Prob_{Y_i | Y_{1:i-1}=x_{1:i-1}}(x_i)
		\\ &=
		\exp\left((\alpha - 1) \, \sum_{i=1}^k \sup_{x_1 \in \cX_1, \dots, x_{i-1}\in \cX_{i-1}} \Dalplr{ \Prob_{X_i | X_{1:i-1}=x_{1:i-1}} } { \Prob_{Y_i | Y_{1:i-1}=x_{1:i-1}} } \right).
	\end{align*}
	The remaining case of $\alpha = 1$ follows by continuity (or by the Chain Rule for KL divergence).

\subsection{Proof of~\cref{lem:sgm-extrema}}\label{app:deferred:sgm-extrema}

Observe that the mixture distribution $(1-q) \cN(0,\sig^2 I_d) + q \left( \cN(0,\sig^2 I_d) \ast \mu\right)$ can be further decomposed as the mixture distribution
\[
(1-q) \cN(0,\sig^2 I_d) + q \left( \cN(0,\sig^2 I_d) \ast \mu\right) = \int \left[ (1-q) \cN(0,\sig^2) + q \cN(z,\sig^2) \right] d\mu(z),
\]
where this integral is with respect to the weak topology of measures. Now argue as follows:
\begin{align*}
	&\;\sup_{\mu \in \cP(R\mathbb{B}_d)} \cD_{\alpha}\left( \cN(0,\sig^2 I_d) \, \big\| \, (1-q) \cN(0,\sig^2 I_d) + q \left( \cN(0,\sig^2 I_d) \ast \mu\right) \right)
	\\ =&\; \sup_{\mu \in \cP(R\mathbb{B}_d)} \cD_{\alpha}\left( \cN(0,\sig^2 I_d) \, \big\| \,  \int \left[ (1-q) \cN(0,\sig^2) + q \cN(z,\sig^2) \right] d\mu(z) \right)
	\\ \leq&\;
	\sup_{z \in \R^d \; : \; \|z\| \leq R} \cD_{\alpha}\left( \cN(0,\sig^2 I_d) \, \big\| \, (1-q) \cN(0,\sig^2 I_d) + q \cN(z,\sig^2 I_d)  \right)
	\\ =&\;
	\sup_{r \in [0,R]} \cD_{\alpha}\left( \cN(0,\sig^2) \, \big\| \, (1-q) \cN(0,\sig^2) + q  \cN(r,\sig^2) \right)
	\\ =&\;
	\cD_{\alpha}\left( \cN(0,\sig^2) \, \big\| \, (1-q) \cN(0,\sig^2) + q \cN(R,\sig^2) \right)
	\\ =&\;
	\cD_{\alpha}\left( \cN(0,(\sig/R)^2) \, \big\| \, (1-q) \cN(0,(\sig/R)^2) + q \cN(1,(\sig/R)^2) \right)
	\\ =&\;
	S_{\alpha}(q,\sig/R).
\end{align*}
Above, the second step is by the quasi-convexity of the R\'enyi divergence (Lemma~\ref{lem:qconvex-rdp}), the third step is by the rotation-invariance of the Gaussian distribution, and the fifth step is by a change of variables. Finally, observe that all steps in the above display hold with equality if $\mu$ is taken to be the Dirac distribution at a point of norm $R$.

\subsection{Proof of~\cref{lem:pabi-contract}}\label{app:deferred:pabi-contract}
	
\par Let $\nu$ be a probability distribution that certifies $\cD_{\alpha}^{(z)}\left( \mu \| \mu' \right)$; that is, $\nu$ satisfies $\cD_{\alpha}^{(z)}(\mu \| \mu') = \cD_{\alpha}(\nu \| \mu')$ and $W_{\infty}(\nu,\mu) \leq z$.
\par We claim that
\begin{align}
	W_{\infty}\left( \phi_{\#}' \nu, \phi_{\#} \mu \right) \leq s + z.
	\label{eq-pf:contraction-iid}
\end{align}
To establish this, first use the triangle inequality for the Wasserstein metric $W_{\infty}$ to bound
\[
W_{\infty}\left( \phi_{\#}' \nu, \phi_{\#} \mu \right)
\leq
W_{\infty}\left( \phi_{\#}' \nu, \phi_{\#}\nu \right)
+
W_{\infty}\left( \phi_{\#}\nu, \phi_{\#}\mu \right).
\]
For the first term, pushforward $\nu$ under the promised coupling of $(\phi,\phi')$ in order to form a feasible coupling for the optimal transport distance that certifies $W_{\infty}( \phi_{\#}' \nu, \phi_{\#}\nu ) \leq s$. For the second term, use the fact that $\phi$ is almost surely a contraction to bound $W_{\infty}( \phi_{\#}\nu, \phi_{\#}\mu ) \leq W_{\infty}( \nu,\mu) \leq z$.
\par Now that we have established~\eqref{eq-pf:contraction-iid}, it follows that $\phi_{\#}'\nu$ is a feasible candidate for the optimization problem $\cD_{\alpha}^{(z+s)} ( \phi_{\#}\mu \| \phi_{\#}' \mu' )$. That is,
\[
\cD_{\alpha}^{(z+s)} ( \phi_{\#}\mu \| \phi_{\#}' \mu' )
\leq
\cD_{\alpha} ( \phi_{\#}' \nu \| \phi_{\#}' \mu' ).
\]
By the quasi-convexity of the R\'enyi divergence (Lemma~\ref{lem:qconvex-rdp}), the post-processing inequality for the R\'enyi divergence (Lemma~\ref{lem:process-rdp}), and then the construction of $\nu$,
\[
\Dalplr{\phi_{\#}'\nu}{\phi_{\#}'\mu'} 
\leq 
\sup_{h \in \supp(\phi')} \Dalplr{h_{\#}\nu}{h_{\#}\mu'} 
\leq 
\cD_{\alpha}(\nu \| \mu')
=
\cD_{\alpha}^{(z)}(\mu \| \mu').
\]
Combining the last two displays completes the proof.

\subsection{Proof of~\cref{thm:iid-lb}}\label{app:lb}

Assume $T \geq \Tthresh$ because once the theorem is proven in this setting, then the remaining setting $T \leq \Tthresh$ follows by the strong composition rule for RDP (Lemma~\ref{lem:comp-rdp-alg}) by equivalently re-interpreting the algorithm $\ISGD$ run once with many iterations as running multiple instantiations of $\ISGD$, each with few iterations. Consider the choice of constants in~\cref{fn:lb-constants}; then for $T \geq \Tthresh$, the quantity in~\eqref{eq:thm-iid-lb} is greater than $0.005$. So it suffices to show that $\ISGD$ does not satisfy $(100,0.005)$-RDP. By the conversion from RDP to DP (Lemma~\ref{lem:rdp-to-dp}), plus the calculation $0.005 + (\log 100)/99 < 0.1$, it therefore suffices to show that $\ISGD$ does not satisfy $(0.1,0.01)$-DP.

\par Consider the construction of datasets $\cX,\cX'$, functions $f$, and processes $W_t,W_t',W_t''$ in \S\ref{sec:lb}. Then, in order to show that $\ISGD$ does not satisfy $(0.1,0,01)$-DP, it suffices to show that $\Prob\left[W_T' \geq 0 \right]
\geq e^{0.1} \Prob[W_T \geq 0] + 0.01$. We simplify both sides: for the left hand side, note that $W_t'$ stochastically dominates $W_t''$ for all $t$; and on the right hand side, note that $\Prob[W_T \geq 0] = 1/2$ by symmetry of the process $\{W_t\}$ around $0$. Therefore it suffices to prove that
\begin{align}
	\Prob\left[W_T'' \geq 0 \right]
	\overset{\text{?}}{\geq} \half e^{0.1} + 0.01.
	\label{eq:iid-lb-toprove}
\end{align}

\par To this end, we make two observations that collectively formalize the Gaussian approximation described in the proof sketch in \S\ref{sec:lb}. Below, let $Y = \sum_{t=T - \Tthresh}^{T-1} Y_t$ denote the total bias in the process $W_t''$, and let $E$ denote the event that both
\begin{itemize}
	\item [(i)] \underline{Concentration of bias in the process $W_t''$:} it holds that $Y \in [1\plusminus \Delta] \cdot \E Y$, for $\Delta = 0.15$.
	\item [(ii)] \underline{No projections in the process $W_t''$:} it holds that $\max_{t \in \{T - \Tthresh, \dots, T\}} W_t'' < D/2$. 
\end{itemize}

\begin{obs}[$E$ occurs with large probability]\label{obs:lb-E}
	$\Prob[E] \geq 0.9$. 
\end{obs}
\begin{proof}
	For item (i) of $E$, note that $B := bY/(\eta L) = \sum_{t=T - \Tthresh}^{T-1} \mathds{1}_{Y_t \neq 0}$ is a binomial random variable with $\Tthresh$ trials, each of which has probability $b/n$, so $B$ has expectation $\E B = b\Tthresh/n = 0.75 bD/(L\eta)$. Thus, by a standard Chernoff bound (see, e.g.,~\citep[Corollary 4.6]{mitzenmacher2017probability}), the probability that (i) does not hold is  at most
	\begin{align*}
		\Prob\Big[ \text{item (i) fails} \Big] 
		=
		\Prob\Big[ B \notin [1 \plusminus \Delta] \cdot \E B \Big]
		\leq 2 \exp\left( - \frac{\Delta^2 \cdot \E B}{3} \right)
		\leq 2 \exp \left( - \frac{0.15^2 \cdot 0.75 \cdot 1000}{3} \right)
		\leq 0.01\,.
	\end{align*}
	\par Next, we show that conditional on (i), item (ii) fails with low probability. To this end, note that (ii) is equivalent to the event that $\sum_{s = T - \Tthresh}^{t-1} (Y_s + Z_s) < D$ for all $t$. Thus because $\sum_{s = T - \Tthresh}^{t-1} Y_s \leq Y \leq (1 + \Delta) \E Y = (1 + \Delta) 0.75 D \leq 0.9 D$ conditional on event (i), we have that 
	\begin{align*}
		\Prob\Big[ \text{item (ii) fails } \Big| \text{ item (i) holds} \Big]
		&\leq 
		\Prob\Bigg[ \max_{t \in \{T - \Tthresh, \dots, T\}} \sum_{s = T -\Tthresh}^{t-1} Z_s \geq 0.1D  \Bigg].
	\end{align*}
	Now the latter expression has a simple interpretation: it is the probability that a random walk of length $\Tthresh$ with i.i.d. $\cN(0,\eta^2 \sig^2)$ increments never surpasses $0.1D$. By a standard concentration inequality on the hitting time of a random walk (e.g., see the application of Doob's Submartingale Inequality on~\citep[Page 139]{williams1991probability}), this probability is at most
	\[
	\cdots
	\leq
	\exp\left( - \frac{(0.1D)^2}{2 \Tthresh \eta^2 \sig^2} \right)
	\leq 
	\exp (-5)
	\leq 0.01\,,
	\]
	\par Putting the above bounds together, we conclude the desired claim:
	\[
	\Prob\Big[ E \Big]
	=
	\Prob\Big[\text{item (i) holds}\Big]
	\cdot 
	\Prob\Big[ \text{item (ii) holds } \Big| \text{ item (i) holds} \Big]
	\geq (1 - 0.01)^2 
	\geq 0.9\,.
	\]
\end{proof}

\begin{obs}[Gaussian approximation of $W_T''$ conditional on $E$]\label{obs:lb-gaussian}
	Denote $Z := \sum_{t=T - \Tthresh}^{T-1} Z_t$. Conditional on the event $E$, it holds that $W_T'' \geq Z + 10/D$.
\end{obs}
\begin{proof}
	By item (ii) of the event $E$, no projections occur in the process $\{W_t''\}$, thus $W_T''
	=
	-D/2 + \sum_{t=T - \Tthresh}^{T-1} (Y_{t} + Z_t) = -D/2 + Y + Z $. By item (i), the bias $Y \geq (1 - \Delta) \E Y = (1 - 0.15) 0.75 D \geq 0.6 D$. 
\end{proof}

Next, we show how to combine these two observations in order to approximate $W_T''$ by a Gaussian, and from this conclude a lower bound on the probability that $W_T''$ is positive. We argue that
\begin{align}
	\Prob[W_T'' \geq 0]
	&\geq
	\Prob[W_T'' \geq 0, E]
	\nonumber
	\\&\geq
	\Prob[Z + D/10 \geq 0 , E]
	\nonumber
	\\ &=
	\Prob[Z + D/10 \geq 0 ] - \Prob[Z + D/10 \geq 0 , E^C]
	\nonumber
	\\ &\geq 
	\Prob[Z + D/10 \geq 0 ] - 0.1,
	\nonumber
\end{align}
where above the second step is by Observation~\ref{obs:lb-gaussian}, and the final step is by Observation~\ref{obs:lb-E}. Since $Z$ is a centered Gaussian with variance $\Tthresh \eta^2 \sig^2 \leq c_{\sig} D^2 = 0.001D^2$, we have
\begin{align}
	\Prob\left[ Z + D/10 \geq 0 \right]
	\geq 
	\Prob\left[ \cN(0.1D, 0.001 D^2) \geq 0 \right]
	= 
	\Prob\left[ \cN(0,1) \geq -0.1/\sqrt{0.001} \right]
	\approx 0.999993.
	\nonumber
\end{align}
\par By combining the above two displays, we conclude that
\begin{align*}
	\Prob\left[ w_T'' \geq 0 \right] 
	\geq 
	0.99993 - 0.1
	\geq \half e^{0.1} + 0.01.
\end{align*}
This establishes the desired DP lower bound~\eqref{eq:iid-lb-toprove}, and therefore proves the theorem.


\section{Non-asymptotics for numerics}\label{app:numerics}

In this paper we have established asymptotic expressions for the privacy loss of $\ISGD$ (and variants thereof). While our results are tight up to a constant factor (proven in \S\ref{sec:lb}), constant factors are important for the numerical implementation of private machine learning algorithms. As such, we describe here how to numerically optimize these constants---first for full-batch noisy gradient descent in~\cref{app:numerics:full-batch} since this case is simplest, then for $\ISGD$ in~\cref{app:numerics:isgd}, and finally for strongly convex losses in~\cref{app:numerics:sc}.

We make two remarks that apply to all of the non-asymptotic bounds in this section. 

\begin{remark}[Full range of parameters]\label{rem:alpha}
	The non-asymptotic bounds in this section hold for all R\'enyi parameters $\alpha \geq 1$ and all noise levels $\sigma > 0$, even though Theorem~\ref{thm:iid-ub} does not. (This is because the proof of Theorem~\ref{thm:iid-ub} uses Lemma~\ref{lem:sgm} to obtain asymptotics, and that lemma only holds for certain $\alpha$ and $\sigma$).
\end{remark}

\begin{remark}[Alternative notions of dataset adjacency]\label{rem:adjacency}
	Differential privacy refers to the indistinguishability of an algorithm's output when run on ``adjacent'' datasets, and thus privacy bounds differ depending on the notion of adjacency between datasets. 
	There are two standard such notions: two datasets are ``replace adjacent'' if they agree on all but one datapoint, and are ``remove adjacent'' if one dataset equals the other plus an additional datapoint. Throughout this paper, we proved bounds for ``replace adjacency''; for ``remove adjacency'', the bounds improve by a small constant factor. Specifically, replace occurences of $2L$ by $L$ in the bounds in this section.\footnote{This factor of $2$ improvement is due to the fact that the $s_t$ term in the Privacy Amplification by Iteration Argument can be improved to $\eta L$ under ``remove adjacency'', as opposed to $2\eta L$ under ``replace adjacency''.}
\end{remark}

\subsection{Full-batch updates}\label{app:numerics:full-batch}
We stop the proof of Theorem~\ref{thm:diam-rdp} before the final asymptotic simplifications which lose small constant factors. Then full-batch GD satisfies $(\alpha,\eps)$-RDP for 
\begin{align}
	\eps
	\leq
	\frac{\alpha}{2 \eta^2 \sig^2} \min\left\{ T \left( \frac{2 \eta L}{n} \right)^2, \min_{\tilde{T} \in \{1,\dots,T\}} \tilde{T} \left( \frac{\tilde{D}}{\tilde{T}} + \frac{2\eta L}{n} \right)^2 \right\}.
	\label{eq:fullbatch-nonasymp}
\end{align}
Here, the outer minimization combines two bounds. The second bound comes from~\eqref{eqn:cyclic-ub-at-sum} by applying the new Privacy Amplification by Iteration argument (Proposition~\ref{prop:pabi-new}) with $\tau = T - \tilde{T}$, $s_t = 2\eta L / n$, and $a_{\tau + t} = \tilde{D}/\tilde{T} + 2\eta L /n$, where $\tilde{D} = D + 2\eta L/ n$. The first bound comes simply applying the original Privacy Amplification by Iteration argument (Proposition~\ref{prop:pabi-original}) there with $s_t = a_t = 2\eta L / n$.

\paragraph*{Numerical computation of privacy bound.} This bound~\eqref{eq:fullbatch-nonasymp} can be efficiently computed by taking $\tilde{T}$ to be $(\tilde{D}n)/(2\eta L)$, modulo a possible floor or ceiling, if this is at most $T$.

\subsection{Small-batch updates}\label{app:numerics:isgd}

We stop the proof of Theorem~\ref{thm:iid-ub} before the final asymptotic simplifications which lose small constant factors. Then $\ISGD$ satisfies $(\alpha,\eps)$-RDP for
\begin{align}
	\eps
	&\leq
	\min_{\substack{\sig_1,\sig_2 \geq 0 \\ \sig_1^2 + \sig_2^2 = \sig^2}} 
	\min\left\{ T Q, \min_{\tilde{T} \in \{1,\dots,T-1\}} \tilde{T}Q + \frac{\alpha D^2}{2\eta^2 \sig_1^2 \tilde{T}}  \right\}
	\nonumber
\end{align}
where $Q = S_{\alpha}( b/n, b\sig_2 / (2L))$. Above, the outermost minimization is from the noise-splitting in Step 1 of the proof of Theorem~\ref{thm:iid-ub}. The next minimization combines two bounds. The latter bound is~\eqref{eq-pf:isgd:nonasymp}, modulo a change of variables $\tilde{T} = T - \tau$. The former bound is by noting that if $\tau = 0$, then the two Contractive Noisy Iterations in the proof of Theorem~\ref{thm:iid-ub} are identical, whereby $\circled{2} =0 $. Now, by swapping the outermost minimizations in the above bound, we may simplify to
\begin{align}
	\eps \leq
	\min\left\{ T \cdot S_{\alpha}\left( \frac{b}{n}, \frac{b\sig}{2 L} \right), \min_{\substack{\sig_1,\sig_2 \geq 0 \\ \sig_1^2 + \sig_2^2 = \sig^2}}  \min_{\tilde{T} \in \{1,\dots,T-1\}} \tilde{T}Q + \frac{\alpha D^2}{2\eta^2 \sig_1^2 \tilde{T}}  \right\}
	\label{eq:iid-nonasymp-2}
\end{align}

\paragraph*{Numerical computation of privacy bound.} This bound~\eqref{eq:iid-nonasymp-2} can be efficiently computed via ternary search over $\sig_1 \in [0,\sig]$, where for each candidate $\sig_1$ we compute $Q$ to high precision using the numerically stable formula in~\citep{rdp_sgm} and then evaluate the resulting minimization over $\tilde{T}$ by setting it to 
$\min(T-1, D/(\eta \sig_1) \sqrt{\alpha /(2Q)})$,
modulo a possible floor or ceiling. In order to solve for the smallest noise $\sig^2$ given an RDP budget, one can use the aforementioned computation of the RDP budget given $\sig^2$ in order to binary search for $\sig^2$ (since the RDP is increasing in $\sig$).

\subsection{Strongly convex losses}\label{app:numerics:sc}
We stop the proof of Theorem~\ref{thm:iid-sc} before the final asymptotic simplifications which lose small constant factors, and optimize\footnote{
	Specifically, rather than simply setting 
	$a_{T} = c^{T-\tau}D$ and the rest to $0$, 
	improve the bound on $\circled{2}$ in~\eqref{eq-pf:isgd:B-sc} to $\alpha/(2\eta^2 \sig_1^2)$ times the minimal value of $
	\sum_{t=\tau + 1}^{T} a_t^2$ when optimizing over $a_{\tau+1},\dots,a_{T}$ subject to $\sum_{t=\tau + 1}^{T} c^{\tau - t} a_t = D$. This has a closed-form solution: $a_t = c^{-t} \beta D$, where $\beta = c^{\tau + 2}(1-c^{-2})/(1 - c^{-2(T - \tau )})$, yielding value $c^{-\tau}\beta D^2$.
}
the parameters $a_t$ in the application of the new Privacy Amplification by Iteration argument (Proposition~\ref{prop:pabi-new-sc}). Then, as argued in~\cref{app:numerics:isgd}, $\ISGD$ satisfies $(\alpha,\eps)$-RDP for
\begin{align}
	\eps 
	&\leq
	\min_{\substack{\sig_1,\sig_2 \geq 0 \\ \sig_1^2 + \sig_2^2 = \sig^2}} 
	\min\left\{ T Q, \min_{\tilde{T} \in \{1,\dots,T-1\}} \tilde{T}Q + \frac{\alpha D^2 (1-c^2)}{2\eta^2 \sig_1^2 (c^{-2 \tilde{T}} - 1)}  \right\} \nonumber
	\\ &=
	\min\left\{ T \cdot S_{\alpha}\left( \frac{b}{n}, \frac{b\sig}{2 L} \right), \min_{\substack{\sig_1,\sig_2 \geq 0 \\ \sig_1^2 + \sig_2^2 = \sig^2}}  \min_{\tilde{T} \in \{1,\dots,T-1\}} \tilde{T}Q + \frac{\alpha D^2 (1-c^2)}{2\eta^2 \sig_1^2 (c^{-2 \tilde{T}} - 1)}   \right\}
	\label{eq:iid-sc-nonasymp}
\end{align}
where $Q = S_{\alpha}( b/n, b\sig_2 / (2L))$.

\paragraph*{Numerical computation of privacy bound.} This bound~\eqref{eq:iid-sc-nonasymp} can be efficiently computed via ternary search over $\sig_1 \in [0,\sig]$, where for each candidate $\sig_1$ we compute $Q$ to high precision using the numerically stable formula in~\citep{rdp_sgm} and then evaluate the resulting minimization over $\tilde{T}$ by using first-order optimality to set it to
\[
\tilde{T} = - \frac{ \log\left( \frac{2 - \beta + \sqrt{(2-\beta)^2 - 4}}{2} \right) }{2 \log c}, \qquad \text { where } \qquad \beta = \frac{(1-c)^2 \alpha D^2 \log c}{Q \eta^2 \sig_1^2},
\]
modulo possibly taking a floor/ceiling on $\tilde{T}$, and taking the minimum with $T-1$. In order to solve for the smallest noise $\sig^2$ given an RDP budget, one can use the aforementioned computation of the RDP budget given $\sig^2$ in order to binary search for $\sig^2$ (since the RDP is increasing in $\sig$).

\par We remark that this bound exploits strong convexity via the contraction factor $c < 1$, which exponentially decays the last term in~\eqref{eq:iid-sc-nonasymp}. This contraction is best when $c$ is minimal, i.e., when the stepsize is $\eta = 2/(M+m)$, and it appears that one pays a price in privacy for any deviation from this choice of stepsize.

	\addcontentsline{toc}{section}{References}
	\bibliographystyle{plainnat}
	\bibliography{dp_sgd}{}

\end{document}